\documentclass[12pt]{article}
\usepackage{enumitem}
\usepackage{soul} 
\usepackage[colorlinks]{hyperref}            
\usepackage{color}
\usepackage{graphicx,subfigure,amsmath,amssymb,amsfonts,bm,epsfig,epsf,url,dsfont}
\usepackage{times}
\usepackage{bbm}      
\usepackage{booktabs}
\usepackage{cases}
\usepackage{fullpage}
\usepackage[small,bf]{caption}
\usepackage{algorithm}
\usepackage{algpseudocode}
\usepackage{natbib}
\usepackage[top=1in,bottom=1in,left=1in,right=1in]{geometry}

\renewcommand{\phi}{\varphi}

\renewcommand{\P}{\mathbb{P}}
\newcommand{\E}{\mathbb{E}}

\newcommand{\R}{\mathbb{R}}

\newcommand{\mI}{\mathrm{I}}

\newcommand{\cK}{\mathcal{K}}

\newcommand{\cB}{\mathcal{B}}
\newcommand{\cN}{\mathcal{N}}
\newcommand{\cM}{\mathcal{M}}
\newcommand{\cF}{\mathcal{F}}

\newcommand{\cE}{\mathcal{E}}

\def\ds1{\mathds{1}}
\renewcommand{\epsilon}{\varepsilon}
\newcommand{\eps}{\epsilon}
\newcommand\Var{{\dsV\text{ar}}\,}

\newcommand{\argmin}{\mathop{\mathrm{argmin}}}

\renewcommand{\tilde}{\widetilde}

\newlength{\minipagewidth}
\newcommand{\beq}{\begin{equation}}
\newcommand{\eeq}{\end{equation}}

\newcommand{\beqa}{\begin{eqnarray}}
\newcommand{\eeqa}{\end{eqnarray}}

\newcommand{\beqan}{\begin{eqnarray*}}
\newcommand{\eeqan}{\end{eqnarray*}}

\def\ba#1\ea{\begin{align*}#1\end{align*}} 
\def\banum#1\eanum{\begin{align}#1\end{align}} 

\def \TF{ \mathcal{T} }
\def \RR {\mathbb R}
\def \Var {\mathrm{Var}}
\def \COV {\mathrm{Cov}}
\def \Sph{\mathbb{S}^{n-1}}
\def \Vol{\mathrm{Vol}}
\def \EE {\mathbb E}
\def \PP {\mathbb P}
\def \Var {\mathrm{Var}}

\def\eps{\varepsilon}
\newtheorem{theorem}{Theorem}
\newtheorem{conjecture}{Conjecture}
\newtheorem{lemma}{Lemma}
\newtheorem{claim}{Claim}

\newtheorem{proposition}{Proposition}

\newcommand{\BlackBox}{\rule{1.5ex}{1.5ex}}  
\newenvironment{proof}{\par\noindent{\bf Proof\ }}{\hfill\BlackBox\\[2mm]}

\begin{document}

\title{Kernel-based methods for bandit convex optimization}

\author{S\'ebastien Bubeck
\and Ronen Eldan
\and 
Yin Tat Lee}
\date{\today}

\maketitle

\begin{abstract}
We consider the adversarial convex bandit problem and we build the first $\mathrm{poly}(T)$-time algorithm with $\mathrm{poly}(n) \sqrt{T}$-regret for this problem. To do so we introduce three new ideas in the derivative-free optimization
literature: (i) kernel methods, (ii) a generalization of Bernoulli convolutions, and (iii) a new annealing schedule for exponential weights (with increasing learning rate). The basic version of our algorithm achieves $\tilde{O}(n^{9.5} \sqrt{T})$-regret, and we show that a simple variant of this algorithm can be run in $\mathrm{poly}(n \log(T))$-time per step at the cost of an additional $\mathrm{poly}(n) T^{o(1)}$ factor in the regret. These results improve upon the $\tilde{O}(n^{11} \sqrt{T})$-regret and $\exp(\mathrm{poly}(T))$-time result of the first two authors, and the $\log(T)^{\mathrm{poly}(n)} \sqrt{T}$-regret and $\log(T)^{\mathrm{poly}(n)}$-time result of Hazan and Li. Furthermore we conjecture that another variant of the algorithm could achieve $\tilde{O}(n^{1.5} \sqrt{T})$-regret, and moreover that this regret is unimprovable (the current best lower bound being $\Omega(n \sqrt{T})$ and it is achieved with linear functions). For the simpler situation of zeroth order stochastic convex optimization this corresponds to the conjecture that the optimal query complexity is of order $n^3 / \epsilon^2$.
\end{abstract}

\section{Introduction}
Derivative-free optimization has a long history, going back at least to \cite{Ros60} (see \cite{CSV09} for more on its history and applications). Perhaps surprisingly, the information-theoretic limits for this problem are not yet understood even for bounded convex functions. In the noiseless case \cite{Pro96} (improving upon a result of \cite{NY83}) shows that $O(n^2 \log(n/\epsilon))$ function value queries are sufficient to find an $\epsilon$-approximate minimizer of a convex function (for comparison it is known that $\Theta(n \log(1/\epsilon))$ gradient queries are necessary/sufficient, \cite{Lev65, New65}). On the other hand with noisy function evaluation the current state of the art in \cite{BLNR15} is that $O(n^{7.5} / \epsilon^2)$ queries are sufficient, and that $\Omega(n^2 / \epsilon^2)$ queries are necessary (this lower bound holds even for linear functions, \cite{DHK08}). An even more difficult scenario (where much less is known) is the {\em robust} setting where an adversary can arbitrarily corrupt an $\epsilon$-fraction of the queries. It is only recently that methods with the optimal $\epsilon$-scaling for the number of queries (i.e., $1/\epsilon^2$) were discovered for the robust setting \cite{BE16, HL16}. However those methods are inherently exponential-time (more precisely \cite{BE16} is $\mathrm{poly}(n \log(1/\epsilon))/\epsilon^2$ for the number of queries and $\exp(\mathrm{poly}(n/\epsilon))$-time while \cite{HL16} is $\log(1/\epsilon)^{\mathrm{poly}(n)}/\epsilon^2$ for both the query and time complexity). We note that in \cite{SV15} it is shown for another model of corrupted queries (namely each query can be adversarially modified by at most $\epsilon$) that the exponential dependency on the dimension is unavoidable for some values of $\epsilon = \Omega(1/\mathrm{poly}(n))$. A key contribution of this paper is to give a polynomial-time method for the robust setting described above. Furthermore we conjecture that a modification of our new algorithm (whose pseudo-code is given on the last page) could need as few as $O(n^3 / \epsilon^2)$ queries, which we conjecture to be optimal even without adversarial noise. Our results hold in the more general context of bandit convex optimization which we describe next in Section \ref{sec:BCO}. We give a primer of our contributions in Section \ref{sec:primer}. More related works are described in Section \ref{sec:related}. Finally the introduction is concluded in Section \ref{sec:open} with some open problems that our work raises.

\subsection{Bandit convex optimization} \label{sec:BCO}
We study adversarial bandit convex optimization on a convex body $\cK \subset \R^n$. It can be described as the following sequential game: at each time step $t=1, \hdots, T$, a player selects an action $x_t \in \cK$, and simultaneously an adversary selects a convex loss function $\ell_t : \cK \rightarrow [0,1]$. The player's feedback is its suffered loss, $\ell_t(x_t)$. The player has access to external randomness, and can select her action $x_t$ based on the history $(x_s, \ell_s(x_s))_{s<t}$. The player's perfomance at the end of the game is measured through the 
regret
$$R_T = \sum_{t=1}^T \ell_t(x_t) - \min_{x \in \cK} \sum_{t=1}^T \ell_t(x) ,$$
which compares her cumulative loss 
to the smallest cumulative loss she could have obtained had she known the sequence of loss functions. Without loss of generality we assume that $\cK$ contains a unit ball\footnote{Since we have not yet made any assumptions on the Lipschitz constant of $\ell_t$ one can simply rescale $\cK$.} and for normalization purposes we assume that the diameter\footnote{The diameter only appears logarithmically in our bound. We choose a concrete upper bound on it only to simplify the upcoming equations.} of $\cK$ is at most $T$. Furthermore without loss of generality we can also assume that the losses $\ell_t$ are $T$-Lipschitz (one can simply restrict to a slightly smaller subset of $\cK$).

Our main contribution is to give the first $\tilde{O}(\mathrm{poly}(n) \sqrt{T})$-regret and $\mathrm{poly}(T)$-time algorithm for bandit convex optimization:
\begin{theorem} \label{th:main}
Algorithm \ref{fig:alg} (pseudo-code on last page) satisfies with probability at least $1-1/T$, for some universal constant $c>0$, \footnote{Throughout the paper, we assume $T>n$, for otherwise Theorem \ref{th:main} is trivially true.}
$$R_T \leq c \ n^{9.5} \log^{7.5}(T) \sqrt{T} .$$ 
Furthermore the algorithm can be modified, at the cost of an additional $\mathrm{poly}(n)$ factor (respectively a $\mathrm{poly}(n) T^{o(1)}$ factor) in the regret, such that each step can be run in $\mathrm{poly}(n \log(T)) T$-time (respectively $\mathrm{poly}(n \log(T))$-time), provided that $\cK$ is a polytope described by $\mathrm{poly}(n)$ constraints whose coeffcients are rational numbers with absolute values of numerators and denominators bounded by $\mathrm{poly}(T)$.
\end{theorem}

We conjecture that in fact a much stronger statement holds true (see Section \ref{sec:open} for more on this conjecture).
\begin{conjecture} \label{conj}
There exists an algorithm such that each step takes $\mathrm{poly}(n \log(T))$-time (under the same assumption on $\cK$ as in Theorem \ref{th:main}) and which achieves $\E R_T = \tilde{O}(n^{1.5} \sqrt{T})$. Furthermore no algorithm can achieve a better regret bound for large $n$ and $T$.
\end{conjecture}

\subsection{Contributions} \label{sec:primer}
Theorem \ref{th:main} is the first $\tilde{O}(\mathrm{poly}(n) \sqrt{T})$-regret and $\mathrm{poly}(T)$-time guarantee for bandit convex optimization. We develop several new ideas to achieve this result. We give a brief summary of these ideas below. 

Let $\cM$ be the set of probability measures on $\cK$, and let $\cF$ be the set of measurable functions from $\cK$ to $\R$. In order to avoid overloading notation we will use the same symbol for a measure $p \in \cM$ and for its density with respect to the Lebesgue measure. For $p \in \cM, f\in \cF$ we denote $\langle p, f \rangle = \int f(x) dp(x) $. A Dirac mass at $x$ is denoted by $\delta_x$.

\subsubsection{Kernel methods} \label{sec:kernelsintro}
A major difficulty of the convex bandit problem compared to the linear bandit case is that there is no obvious unbiased estimator of $\ell_t$ based only on the observation of $\ell_t(x_t)$ (while in the linear case one gets an estimator via a one-point linear regression). We go around this issue as follows. Let us fix a kernel $K : \cK \times \cK \rightarrow [0,\infty)$ such that $\int K(x,y) dx = 1$. With a slight abuse of notation the kernel $K$ acts on probability measures $p \in \cM$ as $Kp(x) = \int K(x,y) dp(y)$ and on functions $f \in \cF$ via the adjoint operator $K^*$ defined by $K^* f(y) = \int f(x) K(x,y) dx$. In words $K^* f$ is a linear combination of functions $K(x, \cdot)$ with weights given by the function values of $f$, and thus one has an obvious unbiased estimator for $K^* f$ based on bandit feedback! More precisely, using $f(x)$ where $x$ was sampled from some probability distribution $q$, one has that $f(x) K(x, \cdot) / q(x)$ is an unbiased estimator of $K^* f$ (since $\int q(x) f(x) K(x, \cdot) / q(x) dx = K^* f$).

By playing a no-regret strategy with the unbiased estimator described above one can hope to control instantaneous regrets of the form $\langle p - \delta_x , K^* f \rangle$ (this represents the regret of playing from $p$ --which would be the distribution recommended by the no-regret strategy-- instead of playing $x$ when the loss is $K^* f$). A key observation is that, by definition of the adjoint, the latter quantity is equal to $\langle K(p - \delta_x), f \rangle$. Since one is interested in controlling the regret when the loss is $f$ (rather than $K^* f$) this idendity suggests that instead of playing a point sampled from $p$ one should play from $Kp$. It then only remains to relate $\langle Kp - \delta_x, f \rangle$ (which is the instantaneous regret of playing from $Kp$ instead of playing $x$ when the loss was $f$) to $\langle K(p - \delta_x), f \rangle$ (which is the term that we hope to be able to control when $p$ comes from a no-regret strategy with the estimator described in the previous paragraph).

The above idea is detailed in Section \ref{sec:kernels} (we use continuous exponential weights as the no-regret strategy). 

\subsubsection{Generalized Bernoulli convolutions} \label{sec:coreintro}
As we just explained in Section \ref{sec:kernelsintro} we want to find a kernel $K$ such that $\langle Kp - \delta_x, f \rangle \lesssim \langle K(p - \delta_x), f \rangle$ for all convex functions $f$ and all points $x \in \cK$. We note that for any $\lambda \in (0,1)$ one has
\begin{equation} \label{eq:whatwewantforakernel}
\langle Kp - \delta_x, f \rangle \leq \frac{1}{\lambda} \langle K(p - \delta_x), f \rangle \Leftrightarrow K^* f(x) \leq (1-\lambda) \langle Kp, f\rangle + \lambda f(x) .
\end{equation}
Leveraging the fact that $f$ is convex we see that a natural kernel to consider is such that $K \delta_x$ is the distribution of $(1- \lambda) Z + \lambda x$ for some random variable $Z$ to be defined. Indeed in this case one has
$$K^* f(x) = \E f((1-\lambda) Z + \lambda x) \leq (1-\lambda) \E f(Z) + \lambda f(x) .$$
Thus this kernel satisfies the right hand side of \eqref{eq:whatwewantforakernel} if $Z$ is defined to be equal to $K p$, that is $Z$ satisfies the following distributional identity, where $X \sim p$,
\begin{equation} \label{eq:defcore1}
Z \; \overset{D}{=} \; (1-\lambda) Z + \lambda X .
\end{equation}
If \eqref{eq:defcore1} holds true we say that $Z$ is the {\em core} of $p$. It is easy to see that the core always exists and is unique by taking $Z = \sum_{k=0}^{+\infty} (1-\lambda)^k \lambda X_k$ where $X_0, X_1, \hdots$ are i.i.d. copies of $X$. Interestingly such random variables have a long history for the special case of a random sign $X$ where they are called {\em Bernoulli convolutions}, \cite{Erd39}. Our notion of core can thus be viewed as a generalized Bernoulli convolution. We refer the reader to \cite{PSS00} for a survey on Bernoulli convolutions, and we simply mention that the main objective in this literature is to understand for which values of $\lambda$ is the random variable $Z$ ``smooth'' (say for instance absolutely continuous with respect to the Lebesgue measure). As we 
will see the smoothness of the core will also be key for us (it will allow to control the variance of the unbiased estimator described in Section \ref{sec:kernelsintro}). In order to avoid the difficulties underlying Bernoulli convolutions we will in fact build a kernel based on a {\em Gaussian core} (which can be viewed as some Gaussian approximation of the real core). These ideas are detailed in Section \ref{sec:highdimkernel}. 

We emphasize that the kernel $K$ proposed above depends on the distribution $p$ which in our application will change over time (this will be the exponential weights distribution). Having an adaptive kernel is key for low regret. Indeed for any fixed kernel there is a tradeoff between making $K^* f$ very smooth (in which case the corresponding estimator will have a small variance) and on the other hand having $K^* f$ faithfully represent where the minimum of $f$ is. As time goes by and the exponential weights distribution focuses on a smaller region of space, the kernel should trade off some smoothness far from this region for more accuracy in the approximation of $f$ by $K^* f$ in this region. Naive ideas such as simply taking a convolution with a fixed Gaussian cannot achieve this tradeoff and could not lead to small regret.

Finally the dimension $1$ case turns out to be special and we were able to design a much simpler kernel for this situation: we replace the core of $p$ by a Dirac at the mean of $p$, and instead of a fixed $\lambda$ we take it to be uniformly distributed in $[0,1]$. The analysis of this kernel is described in Section \ref{sec:dim1} where we prove a slightly better regret bound than the one given by Theorem \ref{th:main} for $n=1$, namely we prove a (pseudo-)regret upper bound of order $\log(T) \sqrt{T}$. 

\subsubsection{Focus region, restart, and annealing schedule}
The high-dimensional algorithm (described in Section \ref{sec:highdimalg}) needs to deal with one more difficulty. In dimension $1$ we will see that our kernelized loss estimator has a controlled variance. On the other hand in higher dimensions the variance will only be controlled within a certain {\em focus region} which depends on $p$, and in particular we can only control the regret with respect to points in this focus region. Taking inspiration from \cite{HL16} we then add a testing condition to the algorithm which ensures that, at any round, if the test succeeds then the optimum is within the focus region, and if the test fails then we have negative regret and thus we can safely restart the algorithm. In order to ensure the negative regret property we devise a new adaptive learning rate for exponential weights: basically each time the covariance of the exponential weights changes scale we increase the learning rate so as to make sure that we can quickly adapt to any movement of the adversary, see Section \ref{sec:standardanalysis} and Section \ref{sec:finalanalysis}.

\subsubsection{Polynomial time version}
In Section \ref{sec:time} we briefly describe how to modify Algorithm \ref{fig:alg} to make it a polynomial-time method. The modification mainly relies on existing results concerning sampling/optimization of approximately log-concave functions, but will also require a few tweaks to the parameters of the algorithm, as well as a slightly different constructions of the kernel and the focus region we alluded to above.

\subsection{Related work} \label{sec:related}
The study of bandit convex optimization was initiated in \cite{Kle04, FKM05}. These papers proved that a gradient descent-type strategy with a one-point estimate of the gradient achieves $\tilde{O}(\mathrm{poly}(n) T^{3/4})$-regret. Without further assumptions on the problem this remained the state of the art bound for a decade, until \cite{BE16} proved via an information theoretic argument that there exists a strategy with $\tilde{O}(\mathrm{poly}(n) \sqrt{T})$-regret (in particular by approximately solving the minimax problem this also gives a $\exp(\mathrm{poly}(T))$-time algorithm). Many subcases of bandit convex optimization were investigated during that decade with no progress on the general problem. Most notably the minimax regret for the linear bandit problem (with the bounded loss assumption) is known to be $\tilde{\Theta}(n \sqrt{T})$ thanks to \cite{DHK08, AHR08, BCK12} (this linear case is especially important in practical applications of bandit algorithms because of its connection to the contextual bandit problem, see \cite{BC12}). Beyond the linear case there were three other subcases of bandit convex optimization with $\sqrt{T}$-regret known before \cite{BE16}: (i) $\tilde{O}(n^{16} \sqrt{T})$-regret in \cite{AFHKR11} for the so-called stochastic case where the losses $\ell_t$ form an i.i.d. sequence, (ii) $\tilde{O}(n^{1.5} \sqrt{T})$-regret in \cite{HL14} for the strongly-convex and smooth case (see \cite{ADX10, ST11, DEK15} for some improvements on the $T^{3/4}$-regret with either only strong convexity or only smoothness), and finally (iii) $\tilde{O}(\sqrt{T})$-regret in \cite{BDKP15} for the case $n=1$ (this paper was the first one to propose the information theoretic approach to control the minimax regret for bandit convex optimization). The first ``explicit" $\sqrt{T}$-regret algorithm for general bandit convex optimization was recently proposed in \cite{HL16}. The drawback of the latter result is that the regret (as well as the time complexity) is exponential in the dimension $n$ (while \cite{BE16} shows that a $\mathrm{poly}(n)$ guarantee is achievable).

As we alluded to in the introduction, a closely related problem is the one of zeroth order stochastic convex optimization: the losses $\ell_t$ form an i.i.d. sequence and one is only interested in the optimization error (also known as simple regret): $r_T = \E \ell_T(x_T) - \min_{x \in \cK} \E \ell_T(x)$ (note that a bound on the cumulative regret $R_T$ implies a bound on the simple regret by taking the center of mass of the points played). One important application of bandit convex optimization is to give algorithms for zeroth order stochastic convex optimization which are robust to some amount of adversarial noise. Without adversarial noise the current state of the art is \cite{BLNR15} which gives a $\tilde{O}(n^{3.25} / \sqrt{T})$-simple regret algorithm, while \cite{Sha13} shows that the simple regret has to be $\tilde{\Omega}(n / \sqrt{T})$ even under the strong convexity assumption. We believe that an appropriate modification of Algorithm \ref{fig:alg} should be robust to some adversarial noise and have a $\tilde{O}(n^{1.5} / \sqrt{T})$-simple regret for any bounded convex function, and furthermore that this might be the optimal guarantee for this problem (see Conjecture \ref{conj}). We also note that the general $\tilde{O}(n^{3.25} / \sqrt{T})$ bound can be improved for various subclasses of convex functions using the known results mentioned above for the bandit optimization setting (e.g., $\tilde{O}(n /\sqrt{T})$ for linear functions or $\tilde{O}(n^{3/2} / \sqrt{T}$) for strongly-convex and smooth functions). Another improvement (which also applies with adversarial noise, though it does not extend to the bandit setting) due to \cite{BP16} is that the bound $\tilde{O}(n /\sqrt{T})$ for linear functions can be generalized to infinitely smooth convex functions (interestingly their algorithm is ``kernel-based'' too, although their version is quite different from ours, and in particular their loss estimator is always a linear function).

\subsection{Open problems} \label{sec:open}
The main open problem that remains is to prove Conjecture \ref{conj} (or otherwise find the optimal dependence on the dimension). The proposed dimension dependency $n^{1.5}$ comes from the following heuristic calculation. Instead of taking the Gaussian core to define the high-dimensional kernel one can take the real core and assume (heuristically) that the core is Gaussian. Furthermore instead of applying Azuma-Hoeffding one can use Bernstein-Freedman, which essentially allows in Lemma \ref{lem:smoothnessK} to remove the term $R_1 R_2$ in $\zeta$ (in this case $\zeta$ would be an upper bound on the variance rather than an upper bound on the magnitude of the loss estimate). Ignoring the whole issue of the focus region (i.e., the fact that we only control the variance within a small region) this leads to a regret scaling in $n^{1.5}$. We also note that the same dimension dependency is obtained in \cite{HL14} for strongly-convex and smooth functions, and there too it seems impossible to improve the dimension dependency without fundamentally new ideas. 

It is quite plausible that Conjecture \ref{conj} is wrong and that in fact a $\tilde{O}(n \sqrt{T})$-regret is attainable for all convex functions. An interesting direction to gain confidence in Conjecture \ref{conj} would be to prove that $\Omega(n^{3/2} \sqrt{T})$-regret is unavoidable. The difficulty there is the following: given a query point $x_t$ the best the adversary could have done is to play a linear function (since this would give a smaller loss at all other points), yet if the player knew that the adversary plays linear functions then she can do one-point linear regression and get a $\tilde{O}(n \sqrt{T})$-regret. Thus to show the lower bound in Conjecture \ref{conj} one needs to quantify precisely the relation between the player's information gain and the non-linearity in the loss (this in turn would allow to write explicitly the adversary's trade-off between loss and information).
\newline

Besides proving Conjecture \ref{conj} there are several opportunities to reduce the current dimension dependency. We essentially lose in the dimension in three places: (i) Gaussian core instead of real core (Section \ref{sec:highdimkernel}), (ii) Hoeffding instead of Bernstein (Section \ref{sec:conc}), and (iii) to prove negative regret when one restarts (Section \ref{sec:finalanalysis}) the focus region (and in particular the value of $\alpha$) is larger than what it should be to merely contain most of the mass of the exponential weights which in turn lead to a larger magnitude for the loss estimate. Improving any of these points seem difficult. For example for (i) (but not (i) and (ii) together) it would be sufficient to show that $\E_{X \sim c} (\lambda |\nabla \log c(X)|)$ is finite for $\lambda$ small enough and $c$ the core of an approximately log-concave measure. Replacing the map $s \mapsto \exp(s)$ by $s \mapsto s^2$ in the previous expression one gets the Fisher information of the core. A lot of machinery has been developed to control the Fisher information of repeated convolution of log-concave random variables (note that the core can be viewed as a sort a repeated convolution), see e.g. \cite{BBN03, JB04}. It would interesting to see if some of those techniques can be used here. We also note that to avoid some of the basic number theoretic obstructions of Bernoulli convolutions one might want to take a randomized value of $\lambda$ in the definition of the core.
\newline

Another natural question that our work raises is whether the focus region (and the restart idea) is really necessary. Perhaps the strategy described in Section \ref{sec:kernels} together with the high-dimensional kernel (Section \ref{sec:highdimkernel}) could be enough to prove Theorem \ref{th:main}. At least for the so-called stochastic case (where $\ell_1, \hdots, \ell_T$ is an i.i.d. sequence) it seems like the restart should not play any role (as we explain in Section \ref{sec:highdimalg} the restart takes care of the situation where the adversary makes us zoom in on a small region and then moves the optimal point far away from this region). A basic question is whether one can prove that the restart condition is never satisfied (with high probability) in the stochastic case. 
\newline

Finally we wonder if one could use gradient descent instead of exponential weights in our kernelized framework. Intuitively in our high-dimensional algorithm (Section \ref{sec:highdim}) the distribution $p_t$ is concentrated around its centroid and $\tilde \ell_t$ is not far from a linear function so that when we multiply $p_t$ by $\exp(-\eta \tilde \ell_t)$ it basically moves the centroid in the direction whose expectation is approximately the gradient. A gradient descent type strategy could be beneficial from a ``ational point of view (for example it would perhaps remove the need to use a log-concave sampler, see Section \ref{sec:time}) and furthermore one might use the many tools that were developed to improve gradient descent for various subclasses of convex functions (e.g. smooth or strongly convex, see Section \ref{sec:related}) and improve the dimension dependency of Theorem \ref{th:main} in those cases.

\section{Kernelized exponential weights} \label{sec:kernels}
The central objects in our strategy are a linear map $K : \cM \rightarrow \cM$, and its adjoint $K^*: \cF \rightarrow \cF$ defined by: for any $p \in \cM, f\in \cF$,
$\langle Kp, f \rangle = \langle p, K^* f \rangle$. We will focus on linear maps which can be written as follows (with a slight abuse of notation, writing $K : \cK \times \cK \rightarrow \R$ for the kernel corresponding to the linear map $K$):
\begin{equation}\label{eq:defkernel}
Kp (x) = \int K(x,y) dp(y) , \; \forall x \in \cK, p \in \cM .
\end{equation}
Here, we assume that for every $y \in \cK$ one has that $K(\cdot, y)$ is a measurable function satisfying $\int_{\cK} K(x,y) dx = 1$. In particular we then have:
$$K^* f(y) = \int_{\cK} f(x) K(x,y) dx, \forall y \in \cK, f \in \cF .$$
We will also need a slightly non-standard notion of the ``square" of $K$, which we define as follows:
$$K^{(2)} p(x) = \int K(x,y)^2 d p(y), \; \forall x \in \cK, p \in \cM .$$
We consider the following strategy, which is a kernelized version of continuous exponential weights with bandit feedback: Let $p_1$ be the uniform measure on $\cK$. For any $t \geq 1$ let $K_t$ be a kernel that depends on $p_t$, which we denote as $K_t := K[p_t]$ (see the result below for more on the map $p \mapsto K[p]$). Then one plays $x_t$ at random from $K_t p_t$, observes $\ell_t(x_t)$, and updates $p_{t+1}$ with the standard continuous exponential weights scheme on the estimated function 
$$\tilde{\ell}_t(y) := \frac{\ell_t(x_t)}{K_t p_t(x_t)} K_t(x_t, y) , \; \forall y \in \cK ,$$
that is
$$p_{t+1}(x) = \frac{p_t(x) \exp\left( - \eta \tilde{\ell}_t(x) \right)}{\int p_t(y) \exp\left( - \eta \tilde{\ell}_t(y) \right) dy}  , \; \forall x \in \cK .$$
Note in particular (see also \eqref{eq:unbiasedloss}) that $\E_{x_t \sim K_t p_t} \tilde{\ell}_t(y) = K_t^* \ell_t(y)$ which one should understand as a coarse approximation of $\ell_t$ (where the coarseness depends on $K_t$). The following result shows that under appropriate conditions on the map $p \mapsto K[p]$ this strategy achieves $\sqrt{T}$-regret. In dimension $1$ we will be able to find such a map that exactly satisfies these conditions (see Section \ref{sec:dim1}) but in higher dimensions (Section \ref{sec:highdim}) the situation is more delicate and we won't apply the theorem below directly. For the sake of simplicity, we focus here on the pseudo-regret:
$$\overline{R}_T = \E \sum_{t=1}^T \ell_t(x_t) - \min_{x \in \cK}\E \sum_{t=1}^T \ell_t(x) .$$

\begin{theorem} \label{th:basic}
Assume that $\cM \ni p \mapsto K[p]$ satisfies the following three conditions. There exists $\epsilon, {\lambda} > 0$ such that for any convex and $T$-Lipschitz function $f \in \cF$, any $x \in \cK$, and any $p \in \cM$,
\begin{equation} \label{eq:cond1}
K[p]^*f(x) \leq (1-\lambda) \langle K[p]p, f \rangle + \lambda f(x) + \epsilon .
\end{equation}
There exists $C > 0$ such that for any $p \in \cM$,
\begin{equation} \label{eq:cond2}
\int \frac{K[p]^{(2)} p(x)}{K[p] p (x)} dx \leq C .
\end{equation}
Finally there exists $L > 0$ such that for any convex and $1$-Lipschitz function $f \in \cF$ and any $p \in \cM$, one has that $K[p]^* f$ is $L$-Lipschitz.

Then the strategy described above satisfies, with $\eta = \sqrt{\frac{2 n \log(L T^3)}{C T}}$,
\begin{equation} \label{eq:basicbound1}
\overline{R}_T \leq \frac{T \epsilon + 2}{\lambda} + \frac{1}{\lambda} \sqrt{2 n C T \log(L T^3)} .
\end{equation}
\end{theorem}

\begin{proof}
Let $x^* \in \argmin_{x \in \cK} \E \sum_{t=1}^T \ell_t(x)$. Note that \eqref{eq:cond1} is equivalent to
$$\langle K[p]p - \delta_{x} , f \rangle \leq \frac{1}{\lambda} \; \langle K[p] (p - \delta_{x}) , f \rangle + \frac{\epsilon}{\lambda}$$
and thus one can write 
\begin{equation} \label{eq:ineq1}
\overline{R}_T = \E \sum_{t=1}^T (\ell_t(x_t) - \ell_t(x^*)) 
 = \E \sum_{t=1}^T \langle K_t p_t - \delta_{x^*}, \ell_t \rangle \leq \frac{T \epsilon}{\lambda} + \frac{1}{\lambda} \; \E \sum_{t=1}^T \langle p_t - \delta_{x^*}, K_t^* \ell_t \rangle.
\end{equation}
Next, we note that the estimated loss $\tilde{\ell}_t$ is an unbiased estimator of $K_t^* \ell_t$ since for any $y \in \cK$,
\begin{equation} \label{eq:unbiasedloss}
\E_{x_t \sim K_t p_t} \tilde{\ell}_t(y) = \E_{x_t \sim K_t p_t} \frac{\ell_t(x_t)}{K_t p_t(x_t)} K(x_t, y) = \int \ell_t(x) K_t(x, y) dx = K_t^* \ell_t(y) .
\end{equation}
Thus, the inequality \eqref{eq:ineq1} can be rewritten as
\begin{equation} \label{eq:ineq2}
\overline{R}_T \leq \frac{T \epsilon}{\lambda} + \frac{1}{\lambda} \; \E \sum_{t=1}^T \langle p_t - \delta_{x^*} , \tilde{\ell}_t \rangle .
\end{equation}
In words, inequality \eqref{eq:ineq2} shows that the pseudo-regret of our strategy is controlled (up to a multiplicative factor $1 / \lambda$) by the pseudo-regret of playing basic continuous exponential weights on the sequence of losses $\tilde{\ell}_1, \hdots, \tilde{\ell}_T$. In particular a straightforward calculation used in standard analysis of exponential weights (see below for more details) gives 
\begin{equation} \label{eq:standanalysis}
\sum_{t=1}^T \langle p_t - \delta_{x^*} , \tilde{\ell}_t \rangle \leq 2 + \frac{n \log(L T^2 \mathrm{diam}(\cK))}{\eta} + \frac{\eta}{2} \sum_{t=1}^T \langle p_t, \tilde{\ell}_t^2 \rangle .
\end{equation}
It only remains to observe that thanks to \eqref{eq:cond2}:
$$\E_{x_t \sim K_t p_t} \langle p_t, \tilde{\ell}_t^2 \rangle = \int K_t p_t(x) p_t(y) \frac{\ell_t(x)^2}{(K_t p_t(x))^2} K_t(x,y)^2 dydx \leq \int \frac{K^{(2)}_t p_t(x)}{K_t p_t(x)} dx \leq C .$$
Combining the above inequality with \eqref{eq:ineq2} and \eqref{eq:standanalysis} easily concludes the proof.

For sake of completeness we now give some details on the derivation of \eqref{eq:standanalysis}. An elementary calculation yields for any $q \in \cM$,
$$\sum_{t=1}^T \langle p_t - q, \tilde{\ell}_t \rangle = \frac{\mathrm{Ent}(q \| p_1) - \mathrm{Ent}(q \| p_{T+1})}{\eta} + \frac{1}{\eta} \sum_{t=1}^T \log \E_{X \sim p_t} \exp \left(- \eta (\tilde{\ell}_t(X) - \E_{X' \sim p_t} \tilde{\ell}_t(X')) \right) .$$
Using that $\tilde{\ell}(x) \geq 0$ for any $x \in \cK$, and that $\log(1+s) \leq s$ and $\exp(-s) \leq 1 - s + \frac{s^2}{2}$ for any $s \geq 0$ one has 
$$\log \E_{X \sim p_t} \exp \left(- \eta (\tilde{\ell}_t(X) - \E_{X' \sim p_t} \tilde{\ell}_t(X')) \right)  \leq \frac{\eta^2}{2} \E_{X \sim p_t} \tilde{\ell}_t(X)^2 .$$
Now let $q$ be the uniform measure on $(1-s) x^* + s\cK$. Then since $\tilde{\ell}_t$ is $L T$-Lipschitz (recall that without loss of generality we assume that $\ell_t$ is $T$-Lipschitz) one has (recall also that we assume $\mathrm{diam}(\cK)\leq T$):
$$\sum_{t=1}^T \langle p_t - \delta_{x^*}, \tilde{\ell}_t \rangle \leq 2 s L T^3 + \sum_{t=1}^T \langle p_t - q, \tilde{\ell}_t \rangle$$
and furthermore $\mathrm{Ent}(q \| p_1) = n \log(1/s)$. This concludes the proof of \eqref{eq:standanalysis} by taking $s = 1/(L T^3)$.
\end{proof}

\section{Construction of a kernel in dimension 1} \label{sec:dim1}
In this section we assume that $\cK = [0,1]$ and let $p \in \cM$ be fixed. The objective is to construct a kernel $K : [0,1] \times [0,1] \rightarrow \R$ which satisfies the three conditions of Theorem \ref{th:basic}. We propose the following simple kernel. Define $\mu = \E_{X \sim p} X$ (we assume that $\mu \geq \epsilon$, the whole argument is easily modified if one instead assumes $\mu \leq 1- \epsilon$) and denote by $[a,b]$ the segment between $a$ and $b$. We set
\begin{equation} \label{eq:kerneldim1}
K(x,y) = \left\{\begin{array}{cc} \frac{\ds1\{x \in [y, \mu]\}}{|y - \mu|} & \text{if} \; |y-\mu| \geq \epsilon , \\ \\  \frac{\ds1\{x \in [\mu-\epsilon, \mu ]\}}{\epsilon} & \text{if} \; |y-\mu| < \epsilon  \end{array}\right .
\end{equation}
and define the linear map $K : \cM \rightarrow \cM$ using equation \eqref{eq:defkernel}. In other words, if $|y-\mu| \geq \epsilon$ then $K \delta_y$ is the uniform distribution on the segment $[y, \mu]$, while otherwise it is the uniform distribution on $[\mu-\epsilon, \mu]$. The adjoint also has a simple description: Using $U$ to denote a uniform random variable in $[0,1]$, we have
$$K^* f(y) = \langle K \delta_y, f \rangle = \left\{\begin{array}{cc} \E \; f(U \mu + (1-U) y) & \text{if} \; |y-\mu| \geq \epsilon , \\ \\  \E \; f(\mu - \epsilon U)  & \text{if} \; |y-\mu| < \epsilon . \end{array}\right. $$
It is clear that if $f$ is $1$-Lipschitz then so is $K^* f$ on $[0,\mu+\epsilon)$ and $[\mu+\epsilon,1]$, and thus with the notation of Theorem \ref{th:basic} one can take\footnote{One needs to adapt the proof of Theorem \ref{th:basic} to deal with the small discontinuity of $K^* f$. In fact since the discontinuity gap at $m+\epsilon$ is smaller than $\epsilon$ it is easy to see that one only needs to replace $T \epsilon$ in \eqref{eq:basicbound1} by $2 T \epsilon$.} $L=1$. Let us now check condition \eqref{eq:cond1}. First observe that if $|x-\mu| < \epsilon$ then the $T$-Lipschitzness of $f$ implies \eqref{eq:cond1} with $\lambda=1$ and with $T \epsilon$ instead of $\epsilon$. On the other if $|x-\mu| \geq \epsilon$ we use the convexity of $f$ as follows:
$$K^*f(x) = \E \; f(U \mu + (1-U) x) \leq \frac{f(\mu) + f(x)}{2} \leq \frac{\langle K p, f \rangle + f(x)}{2} + \epsilon,$$
where the second inequality follows from Jensen's inequality and the fact that the mean $\tilde{\mu}$ of $Kp$ verifies $|\mu-\tilde{\mu}| \leq 2 \epsilon$. This directly implies \eqref{eq:cond1} with $\lambda=1/2$. Thus it only remains to check \eqref{eq:cond2}. For this we use $K(x,y) \leq \frac{1}{\max(|x - \mu|, \epsilon)}$ which implies $K^{(2)} q(x) \leq \frac{K q(x)}{\max(|x - \mu|, \epsilon)}$ and in particular
$$\int \frac{K^{(2)} p(x)}{K p (x)} dx \leq \int \frac{1}{\max(|x - \mu|, \epsilon)} dx \leq 2 (1+\log(1/\epsilon)) .$$ 
Thus with $\epsilon = 1/T^2, L=1, \lambda=1/2$ and $C = 2 \log(e T^2)$ one finally obtains the following upper bound on the pseudo-regret of our kernel-based strategy with the kernel described in \eqref{eq:kerneldim1}:
$$\overline{R}_T \leq 12 \log(T) \sqrt{T} .$$

\section{The high-dimensional case} \label{sec:highdim}
As we already mentioned the case $n \geq 2$ turns out to be much more challenging than the one-dimensional case. Here we won't be able to use Theorem \ref{th:basic} directly (however we will verify similar properties to those mentioned in Theorem \ref{th:basic}). In this section we describe the high-dimensional kernel and the high-dimensional algorithm. In Section \ref{sec:highdimanalysis} we give the regret analysis of the algorithm and in Section \ref{sec:time} we explain how to modify the algorithm to make it polynomial-time.

Let us first introduce a few additional notations. We denote by $\mu(p)$ and $\mathrm{Cov}(p)$ the mean and covariance of $p$, and $\cE_p(r) := \{x \in \R^n : \|x - \mu(p)\|_{\mathrm{Cov}(p)^{-1}} \leq r\}$ where for a positive semidefinite matrix $A$ we denote $\|x\|_A := \sqrt{x^{\top}A x}$. We say that $p$ is $\epsilon$-approximately log-concave if there exists a log-concave function $f$ such that for any $x$, $\epsilon f(x) \leq p(x) \leq \frac{1}{\epsilon} f(x)$. Also for a function $f : \Omega \rightarrow \R$ we denote $f^* = \min_{x \in \Omega} f(x)$.

\subsection{The high-dimensional kernel} \label{sec:highdimkernel}
We describe here our proposed kernel map $p \mapsto K[p]$ which depends on two parameters $\epsilon \in (0,1)$ and $\lambda \in (0,1/2)$ to be specified later (eventually $\epsilon$ will be a small numerical constant and $\lambda$ will be $\tilde{O}(1/\mathrm{poly}(n))$). Let us fix a measure $p$ and let 
$$
c[p] = \cN \left (\mu(p), \frac{\epsilon^2}{n \log(T)} \frac{\lambda}{2-\lambda} \mathrm{Cov}(p) \right )
$$ 
be the {\em Gaussian core} of $p$ (this terminology will be explained in Section \ref{sec:convdom}). The linear map $K[p]$ is then defined by: for any $q \in \cM$, $K[p]q$ is the distribution of $(1-\lambda) C + \lambda X$ where $C \sim c[p]$ and $X \sim q$. In other words,
$$
K[p] q \; \overset{D}{=} \; (1 - \lambda)c[p] + \lambda q.
$$
We note that $K[p]q$ is not necessarily supported on $\cK$ and this will lead to a minor technical issues.

In Section \ref{sec:convdom} we prove the first key property of this kernel map which is that for an $(1/e)$-approximately log-concave $p$, $K[p]p$ convexly dominates\footnote{Recall that a measure $p$ convexly dominates a measure $q$ if for any convex function $f$, one has $\langle q, f\rangle \leq \langle p, f \rangle$.} $c[p]$ (approximately). We conclude the study of $K[p]$ in Section \ref{sec:smoothprop} with its smoothness properties when $p$ is appropriately {\em truncated}.

\subsubsection{Convex domination} \label{sec:convdom}

The goal of this section is to prove the following lemma.
\begin{lemma} \label{lem:convexdommain}
Let $p \in \cM$ be an $(1/e)$-approximately log-concave measure supported on a convex body $\cK$ of diameter at most $T$. Let $f:\RR^n \to [0, \infty)$ be a convex function satisfying $f(x) \in [0,1]$ for all $x  \in \cK$ and such that $f$ is non-negative and $T$-Lipschitz on $\RR^n$. Then,
\begin{equation} \label{eq:convexdom}
\langle c[p], f\rangle \leq \langle K[p]p, f \rangle + \frac{1}{T^2} .
\end{equation}
\end{lemma}

Our first step to prove \eqref{eq:convexdom} is the following result proven in the appendix:
\begin{lemma} \label{lem:convexdom}
Let $p$ be an isotropic $(1/e)$-approximately log-concave measure, and let $r$ be a centered measure supported on $\left \{x \in \R^n : |x| \leq \tfrac{1}{80 e} \right \}$. Then one has that $r$ is convexly dominated by $p$ (i.e., for any convex function $f$, $\langle r, f\rangle \leq \langle p, f \rangle$). 
\end{lemma}
We take $r[p] = \cN \left (\mu(p), \frac{\epsilon^2 }{n \log(T)} \mathrm{Cov}(p) \right )$ (we think of $r[p]$ as some sort of Gaussian approximation of $p$). One cannot apply Lemma \ref{lem:convexdom} directly to $r[p]$ since its support is all of $\R^n$. However it is easy to see that, if one chooses
$$
\epsilon = \frac{1}{80e \cdot 20},
$$
then by Lemma \ref{lem:convexdom}, we have for any non-negative convex function $g$,
$$
\langle r[p],  \tilde g \rangle \leq \langle p, \tilde g \rangle \leq \langle p, g \rangle
$$
where $\tilde g(x) := g(x) \ds1\{x \in \cE_p(1/(80e) \}$. Moreover, an application of Lemma \ref{lem:GaussianTail} gives that (provided that $g$ is $T$-Lipschitz and such that $g(\mu(p)) \in [0,2]$)
$$
\langle r[p],  g - \tilde g \rangle \leq \frac{T \eps}{T^3 \sqrt{n \log T}} \leq \frac{1}{T^2}
$$
where we have used the fact that $\mathrm{diam}(\cK) \leq T$ which implies that $\|\COV(r[p])^{1/2} \|_{\mathrm{OP}} \leq \frac{T \eps}{\sqrt{n \log T}}$. Thus, we have that
\begin{equation} \label{eq:risdominated}
\langle r[p], g\rangle \leq \langle p, g \rangle + \frac{1}{T^2} .
\end{equation}
Next we recall the notion of the {\em core} of a distribution introduced in Section \ref{sec:coreintro}: we say that $q'$ is the core of $q$ if the following distributional equality is satisfied, where $X \sim q$, $Y \sim q'$,
$$Y \; \overset{D}{=} \; (1-\lambda) Y + \lambda X .$$
A key observation is that the core of a Gaussian is a Gaussian with smaller variance, more precisely for $q=\cN(0, \mI_n)$ one has $q' = \cN\left(0, \frac{\lambda}{2-\lambda} \mI_n\right)$. In particular we see that $c[p]$ is the core of $r[p]$ (since $r[p]$ is a Gaussian approximation of $p$ this justifies the terminology of Gaussian core of $p$ for $c[p]$). In other words,
\begin{equation}\label{eq:cprp}
c[p] \; \overset{D}{=} \; (1-\lambda) c[p] + \lambda r[p]
\end{equation}

\begin{proof}[Proof of lemma \ref{lem:convexdommain}]	
Observe that the function
$$
g(x) := \E_{ C \sim c[p] } f((1-\lambda)C + \lambda x)
$$
is convex, $T$-Lipschitz and $g(\mu(p)) \in [0,2]$. Thus, by equation \eqref{eq:risdominated} we have
$$\langle c[p], f \rangle \stackrel{\eqref{eq:cprp}}{=} \E_{R \sim r[p]} g(R) \stackrel{\eqref{eq:risdominated}}{\leq} \frac{1}{T^2} + \E_{X \sim p} g(X) = \frac{1}{T^2} + \langle K[p]p , f \rangle . $$
\end{proof}

\subsubsection{Smoothness properties of $K$} \label{sec:smoothprop}
Observe that $K[p](x,y) (= (K[p] \delta_y)(x))$ is the density at $x$ of $(1-\lambda) C + \lambda y$, where $C \sim c[p]$, and thus:
$$K[p](x,y) = c[p]\left( \frac{x - \lambda y}{1-\lambda}\right) (1- \lambda)^{-n} .$$
We now prove a simple but useful lemma.
\begin{lemma} \label{lem:smoothnessK}
Let $R_1, R_2>0$ and $x \in \cE_p(R_1), y, y' \in \cE_p(R_2)$. Then one has
$$\frac{c[p]\left( \frac{x - \lambda y}{1-\lambda}\right)}{c[p]\left( \frac{x - \lambda y'}{1-\lambda}\right)} \leq \zeta , \;\; \text{and} \;\; \frac{\| \nabla_y c[p]\left( \frac{x - \lambda y}{1-\lambda}\right)\|_{\COV(p) }}{c[p]\left( \frac{x - \lambda y'}{1-\lambda}\right)} \leq \zeta' ,$$
where
$$\zeta = \exp\left(\frac{4 n \log(T)}{\epsilon^2} (R_1 R_2 + \lambda R_2^2) \right) , \;\; \text{and} \;\;  \zeta'= \frac{8 n \log(T)}{\epsilon^2} (R_1 + \lambda R_2) \zeta.$$
\end{lemma}

\begin{proof}
The proof is straightforward. Simply note that $\frac{c[p]\left( \frac{x - \lambda y}{1-\lambda}\right)}{c[p]\left( \frac{x - \lambda y'}{1-\lambda}\right)}$ is equal to
\begin{align*}
& \exp\left(\frac{n \log(T) (2-\lambda)}{2 \epsilon^2 \lambda (1-\lambda)^2} (\|x - \mu(p) - \lambda (y' - \mu(p)) \|_{\mathrm{Cov}(p)^{-1}}^2 - \|x - \mu(p)- \lambda (y - \mu(p))\|_{\mathrm{Cov}(p)^{-1}}^2) \right) , \\
& \leq \exp\Bigg(\frac{n \log(T) (2-\lambda)}{2 \eps^2 \lambda (1-\lambda)^2} \big(2 \lambda \|x - \mu(p)\|_{\mathrm{Cov}(p)^{-1}} \|y-y'\|_{\mathrm{Cov}(p)^{-1}} \\
& \hspace{2.5in} + \lambda^2 (\|y-\mu(p)\|_{\mathrm{Cov}(p)^{-1}}^2 + \|y'-\mu(p)\|_{\mathrm{Cov}(p)^{-1}}^2) \big) \Bigg) ,
\end{align*}
and that
$$\left \| \nabla_y c[p]\left( \frac{x - \lambda y}{1-\lambda}\right) \right \|_{\COV(p)} = \frac{n \log(T) (2-\lambda)}{\epsilon^2 (1-\lambda)^2} \ \|x - \mu(p) - \lambda (y - \mu(p)) \|_{\mathrm{Cov}(p)^{-1}} \ c[p]\left( \frac{x - \lambda y}{1-\lambda}\right) ,$$
and use the assumption that $\lambda \in (0,1/2)$.
\end{proof}

A straightforward consequence of this lemma is the following result on the smoothness properties of the loss estimator.
\begin{lemma} \label{lem:smoothnesstildeell}
Assume that $p$ is such that $\mathrm{supp}(p) \subset \cE_p(R_2)$. Let $x \in \cE_p(R_1)$, and let $\ell : \cK \rightarrow [0,+\infty)$ be defined by $\ell(y) = \frac{K[p](x,y)}{K[p]p(x)}$. Then one has that $\ell |_{\cE_p(R_2)}$ takes values in $[0,\zeta]$ and is $\zeta'$-Lipschitz in $\|\cdot\|_{\COV(p_t)^{-1}}$ (where $\zeta$ and $\zeta'$ are defined as in Lemma \ref{lem:smoothnessK}).
\end{lemma}

\subsection{The high-dimensional algorithm} \label{sec:highdimalg}
A major difficulty of the high-dimensional setting is that, on the contrary to the one-dimensional situation, we could not find a kernel for which the estimate $\tilde{\ell}_t(x)$ is controlled for {\em all} points $x \in \cK$ (in fact we believe that such a kernel does not exist). Instead, as explained in Lemma \ref{lem:smoothnesstildeell}, one can control the variance (and in fact even the magnitude) of $\tilde{\ell}_t$ only a small enough ellipsoid $\cE_{p_t}(R)$ for some $R$ to be defined. In order to enforce the exponential weights distribution to be contained in such an ellipsoid and also to somehow acknowledge the fact that the loss estimates outside of this region are not reliable, we truncate the loss estimate outside of a certain {\em focus region} $F_t \subset \cK$ (defined below). Furthermore to make the analysis as clean as possible we want to ignore the possibility that the algorithm plays an atypical point. As we will see the probability of playing outside of $$\Omega_t : = \cK \cap \cE_{p_t}(10 n \alpha \lambda + 20 \sqrt{\lambda} \epsilon)$$ (for some $\alpha \geq 1$ defined below) will be smaller than $1/T^2$. If $x_t$ is atypical (that is $x_t \not\in \Omega_t$) then we will simply set the loss estimate to be $0$ (note that with high probability this does not change the behavior of the algorithm). Thus we finally define $\tilde{\ell}_t$ by
$$\tilde{\ell}_t(y) := \left\{\begin{array}{cc} \frac{\ell_t(x_t) \ds1\{x_t \in \Omega_t\}}{K_t p_t(x_t)} K_t(x_t, y) , & 
\text{if} \; y \in F_t
, \\ \\  +\infty & \text{otherwise} , \end{array}\right.  $$
We will take $F_t$ large enough so that it contains most of the mass of $p_t$, yet small enough so that the loss estimator is well-behaved. We now observe that this truncation induces a significant complication: a priori we do not control anymore the regret with respect to points outside of the focus region. This is where the restart idea comes into play. First, it will be useful to define, 
$$
\ell_t^{ext}(x) := \max \left (\sup_{h \in H} h(x), ~ \inf_{y \in \cK} \ell_t(y) \right )
$$
where $H$ is the family of linear functions $h$ satisfying (i) $h(y) \leq \ell_t(y)$ for all $y \in \cK$ and (ii) there exists $x_0 \in \mathrm{int}(\cK)$ such that $h(x_0) = \ell_t(x_0)$. In other words, we can think of $\ell_t^{ext}$ as the convex extension of $\ell_t$ to $\RR^n$. Next, we define
$$
\tilde{L}_t = \sum_{s=1}^t \tilde{\ell}_s \mbox{ and } L_t = \sum_{s=1}^t K[p_s]^* \ell_s^{ext}.
$$
Observe that as long as $\min_{x \in \partial F_t} \tilde{L}_t$ is significantly larger than $\min_{x \in F_t} \tilde{L}_t$, we know (by concentration of $\tilde{L}_t$ around $L_t$ --which is yet to be proven--, and by convexity of $L_t$) that the minimum of $L_t$ on $\cK$ is also in $F_t$, and thus controlling the regret with respect to points in $F_t$ is sufficient. On the other hand if this is not the case then it means that the adversary made us focus on the region $F_t$, and then later on moved the optimum outside of this region. In particular we can hope to get {\em negative regret} with respect to any fixed point. This is where we need a last idea: we will ensure that each time the region $F_t$ is updated we also increase the learning rate $\eta$ in the exponential weights, so that if a point in $\partial F_t \setminus \partial F_{t-1}$ suddenly becomes very good (in the sense that it has small losses) at some later time, our exponential weights distribution will quickly focus on it. We instantiate this idea as follows. The focus region is initialized at $F_1 = \cK$. For $t \geq 1$ let $A_t$ be the following event, for some $\alpha \geq 1$,
\begin{equation}\label{eq:shrinkFt}
\mathrm{Vol}(F_t \cap \cE_{p_{t+1}}(\alpha)) \leq \frac{1}{2} \mathrm{Vol}(F_t) ,
\end{equation}
If $A_t$ occurs then we update the focus region and we increase (multiplicatively) the learning rate by $(1+\gamma)$, that is we set $\eta_{t+1} = (1+\gamma \ds1\{A_t\}) \eta_t$. The focus region is updated as follows:  
$$
F_{t+1} = F_{t} \cap \cE_{p_{t+1}}(\alpha).
$$

With the time-dependent learning rate we modify the the exponential weights distribution $p_t$ as follows: let 
$$
p_t(x) = \frac{1}{Z_t} \exp(- Q_{t-1}(x)) ,
$$ 
where 
$$
Q_t = \sum_{s=1}^t \eta_s \tilde{\ell}_s - q ,
$$ 
and where $q$ is chosen so that $\min_{x \in F_t} Q_t(x) = 0$. The point $x_t$ played at round $t$ is chosen as follows: we draw a point $X$ at random from $K[p_t]p_t$ and set $x_t = X$ when $X \in \cK$; otherwise we choose $x_t$ to be an arbitrary point in $\cK$. Finally the restart condition is as follows, for some $\beta >0$.
\begin{align*}
& \text{if there exists} \ x \in \partial F_{t+1} \cap \mathrm{int}(\cK) \ \text{such that} \ \eta_1 (\tilde{L}_t(x) - \tilde{L}_t^*) \leq \beta \\ 
& \text{then restart the algorithm, i.e. act as if time step}\ t+1 \ \text{was time step} \ 1 \ \text{and replace} \ T \ \text{by} \ T-t .
\end{align*}

\subsubsection{Assumptions about the parameter values}
The algorithm has four parameters, $\eta_1$, $\alpha$, $\beta$, and $\gamma$ (in addition to the kernel map parameters $\epsilon$ and $\lambda$). The exact values for the parameters will be determined later on. However, we will make the following assumptions about our parameters, which will later be verified by our choices. 

\begin{align}
	\text{(i) } ~& \eta_1, \lambda, \beta, \gamma < 1/2, \alpha \geq 1 \text{ and } n \alpha \sqrt{\lambda} \leq 1. \nonumber \\ 
	\text{(ii) }& 0 < \epsilon < 1/e. \label{eq:paramassump} \\
	\text{(iii) }& \max \left ( (\eta_1 \sqrt{T})^{-1}, \lambda^{-1}, \gamma^{-1}, \alpha, \epsilon^{-1} \right ) \leq C' n^{C} \log(T)^C \leq T^{1/2}, \nonumber
\end{align}
where $C',C>0$ denote universal constants which can be taken to be $C=8$ and $C' = 2^{30}$. \\

We will take $\alpha \approx n^2 \log^2(T)$ (this ensures that $\cE_{p_t}(\alpha)$ contains most of the mass of $p_t$, and more importantly that points on the boundary of $\cE_{p_t}(\alpha)$ have a very large $Q$-value), $\gamma \approx 1/ (n \log(T))$ (this will ensure that $\eta_T / \eta_1 \approx 1$), $\beta$ of constant order, and finally $\eta_1^{-1} \approx \sqrt{T n \log(T)}$. The key parameter $\lambda$ of the kernel will be set small enough so that $\zeta$ (hence the bound for $\tilde \ell_t$ given by Lemma \ref{lem:smoothnesstildeell}) will be a numerical constant, namely $\lambda \approx \tfrac{1}{n^4 \alpha^2 \log^2(T)} \approx n^{-8} \log^{-6} T$.

\section{Analysis of the high-dimensional algorithm} \label{sec:highdimanalysis}
Our first order of business is to understand the concentration properties of $\tilde{L}_t$ and $Q_t$, which will in particular show that $p_t$ is $(1/e)$-approximately log-concave, see Section \ref{sec:conc}. Then we adapt the standard analysis of exponential weights to our time-dependent learning rate in Section \ref{sec:standardanalysis}. We conclude the regret analysis in Section \ref{sec:finalanalysis}. Before all of this we introduce some defintions in Section \ref{sec:defs} and we make some simple useful observations in Section \ref{sec:usefulobs}.

\subsection{Some central definitions} \label{sec:defs}
Let $\tau$ be the minimum between $T$ and the first time at which the algorithm restarts. Let $\tau_1, \hdots, \tau_N$ be the times in $\{1,\hdots, \tau\}$ at which we increase the learning rate, that is $\eta_{\tau_i+1} = (1+\gamma) \eta_{\tau_i}$. 

Next, we consider the events
$$
{\tiny{ B_t := \left\{\max\left(\Delta_t^{(1)}, \dots, \Delta_t^{(4)} \right) \leq 1 \right\}}}, ~ \forall t \leq \tau ,
$$
where
$$
\Delta^{(1)}_t := \eta_1 \left | \sum_{s=1}^t \left (\langle p_s, K_t[p_s]^* \underline{\ell}_s \rangle - \ell_s(x_s) \right ) \right |, ~~~ \Delta_t^{(2)} := \max_{y \in F_t} \eta_1 \left| \tilde{L}_t(y) - L_t(y) \right| ,
$$
$$
\Delta_t^{(3)} :=  \max_{y \in F_t} \left| \sum_{s=1}^t \eta_s \left (\tilde \ell_s(y) - K_s[p_s]^* \ell_s^{ext}(y) \right ) \right |, ~~~ \Delta^{(4)}_t := \eta_1 \left | \sum_{s=1}^t \left \langle p_s, \tilde \ell_s - K_s[p_s]^* \ell_s^{ext} \right \rangle  \right | ,
$$
and 
$$
\underline{\ell}_t(y) := \ell_t(y) \ds1\{y \in \Omega_t\}.
$$ 
A central definition will be the following ``fault" stopping time:
$$
\TF := \inf \Bigl \{t \leq \tau; ~ B_t \mbox{ does not hold or } x_t \notin \Omega_t  \Bigr \} \wedge \tau.
$$
Note that, in particular, we have
\begin{claim}\label{claim:log_concave}
For all $t \leq \TF$ one has that $p_t$ is $(1/e)$-approximately log-concave.
\end{claim}
\begin{proof}
Fix $t < \TF$. By the convexity of $\ell_s^{ext}$ for all $s \leq t$, we have that $K_s[p_s]^* \ell_s^{ext}$ is convex. Thus, $\Delta_t^{(3)} \leq 1$ implies that there exists a convex function $g_t$ such that $|g_t(y) - Q_t(y)| \leq 1$ for all $y \in F_t$. Since $Q_t$ is supported on $F_t$, we have that $p_{t+1} \propto \exp(-Q_t)$ is $(1/e)$-approximately log-concave.
\end{proof}

Our analysis will be carried out in two central steps:
\begin{proposition} \label{prop:mainstep}
We have, almost surely
$$	
\max_{x \in \cK}  \sum_{t=1}^{\TF} (\ell_t(x_t) - \ell_t(x)) \leq 
\begin{cases}
C n^{9.5} \log^{7.5}(T) \sqrt{T} & \TF = \tau = T \\
0 & \TF = \tau < T \\
\TF & \mbox{otherwise}
\end{cases}
$$
for a universal constant $C>0$ (we can take $C=6^{90}$).
\end{proposition}
and,
\begin{proposition} \label{prop:conc}
We have $\P(\TF < \tau) < 2/T^2$. 
\end{proposition}

Let us now see why a combination of these two facts establishes the final regret bound, proving Theorem \ref{th:main}. 

\begin{proof} [Proof of Theorem \ref{th:main}]
Let us first denote by $T_1, T_2,...,T_k$ the times in which the algorithm restarts, hence, we set $T_1 = \tau$; in case $\tau < T$ we run the algorithm again which provides another restart time $\tau$ and we set $T_2 - T_1 = \tau$ and so on, until reaching $T_k = T$. Moreover, denote by $\TF_1, \dots, \TF_k$ the respective values of $\TF$ for each of these rounds. Finally set $T_0 = 0$.

Let $E$ be the event that $\TF_i = T_i$ for all $1 \leq i \leq k$. An application of Proposition \ref{prop:mainstep} gives that whenever $E$ holds, we have
$$
\max_{x \in \cK}  \sum_{t=1}^{T} (\ell_t(x_t) - \ell_t(x)) \leq \sum_{i=1}^k \max_{x \in \cK}  \sum_{t=1}^{T_i - T_{i-1}}  (\ell_t(x_t) - \ell_t(x))  \leq C n^{9.5} \log^{7.5}(T) \sqrt{T}.
$$
Finally, using Proposition \ref{prop:conc}, the fact that $k \leq T$, and a union bound gives
$$
1 - \P(E) \leq T \P (\TF_1 < \tau) \leq \frac{2}{T}.
$$
Combining the two last displays completes the proof.
\end{proof}

\subsection{Some simple facts} \label{sec:usefulobs}

In this section, we establish several facts about $F_t$ and $\tilde \ell_t$:
\begin{enumerate}[label=(\roman*)]
\item
We will show that $F_t$ is contained in the ellipsoid $\cE_{p_t}(10 n \alpha)$.
\item
We will show that the volume of $F_t$ is bounded from below by $T^{-Cn}$ and so is $\det \COV (p_t)$.
\item
The bound on the volume of $F_t$ will yield respective bounds $N \leq C n \log T$ and, with an appropriate choice of the constant $\gamma$, we will get $\eta_\tau \leq e \eta_1$.
\item
Finally, we will show that $\tilde \ell_t$ is upper-bounded by a constant inside $F_t$ and its gradient is bounded in $\| \cdot \|_{\COV(p_t)}$-norm by a power of $T$.
\end{enumerate}

We begin with,
\begin{claim}
For every $t \in [\tau]$ one has
\begin{equation}\label{eq:FtInEllipsoid}
F_t \subset \cE_{p_t}(10 n \alpha)
\end{equation}
\end{claim}
Observe that at $t=1$ this is well-known (see e.g., \cite[Section 1.10]{MR1008717}). On the other hand for $t>1$ we use the following simple lemma:
\begin{lemma} \label{lem:largerintersection}
Let $\cK$ be a convex body and $\cE$ an ellipsoid centered at the origin. Suppose that $\mathrm{Vol}(\cK \cap \cE) \geq \tfrac{1}{2} \mathrm{Vol}(\cK)$. Then $\cK \subset 10 n \cE$.
\end{lemma}
\begin{proof}
	By applying a linear transformation, we can clearly assume that $\cE$ is the unit ball. Let us prove the contrapositive and assume that there is a point $x \in \cK$ with $|x| > 10n$. Denote $s_i = \frac{2i}{10n}$, $i=1,..,5n$ and consider the sets $\cK_i = (1 - s_i) (\cE \cap \cK) + s_i x$. 
	
	Note that those sets are disjoint. Indeed, the intervals $(1 - s_i) [-1,1] + |x| s_i $ are disjoint, which implies that the projections of the ellipsoids $(1 - s_i) \cE + s_i x$ onto the span of $x$ are disjoint. So, we have
	$$
	\mathrm{Vol}(\cK) \geq \sum_{i=1}^{5 n} \mathrm{Vol}(\cK_i) = \sum_{i=1}^{5 n} (1 - s_i)^{n} \mathrm{Vol}(\cE \cap \cK) \geq 2 \mathrm{Vol}(\cE \cap \cK) ,
	$$
	which concludes the proof.
\end{proof}

Next, define
$$
y_t = \argmin_{x \in F_t} \tilde L_t(x).
$$
Moreover, for the sake of the next claim we will need to set
$$
\beta = 4.
$$
The following fact is a simple consequence of the restart condition.
\begin{claim}
For every $t < \TF$ we have 
\begin{equation}\label{eq:Ftcontainsball}
B \left (y_t, \frac{1}{T^2} \right ) \cap \cK \subset F_t.
\end{equation}
\end{claim}
\begin{proof}
Assume otherwise, hence assume there exists $x \in \partial F_t \cap \mathrm{int}(\cK)$ such that $d(x,y_t) \leq \frac{1}{T^2}$. Then by the definition of $\TF$ and by the assumption that $\ell_s$ is $T$-Lipschitz for all $s \in [T]$, which implies that $L_t$ is $T^2$-Lipschitz, we have that (since $\Delta_t^{(2)} \leq 1$)
$$
|\tilde L_t(x) - \tilde L_t(y_t)| \leq |L_t(x) - L_t(y_t)| + |\tilde L_t(x) - L_t(x)| + |\tilde L_t(y_t) - L_t(y_t)| \leq \frac{3}{\eta_1} 
$$
It follows that the restart condition holds true, which is a contradiction to $t < \TF$.
\end{proof}

As a consequence, we get:
\begin{claim} \label{claim:volFt}
For all $t < \TF$ we have
\begin{equation}\label{eq:volFt}
\Vol(F_t) \geq  T^{- 3n} \omega_n \geq T^{-4n}.
\end{equation}
where $\omega_n$ is the volume of the unit Euclidean ball in $\RR^n$.
\end{claim}
\begin{proof}
By \eqref{eq:Ftcontainsball}, we deduce that, with $B = B \left (y_t, 1/T^2 \right )$,
$$
\Vol(F_t) \geq \Vol(B \cap \cK).
$$
Next, by assumption we have that $\tilde B \subset \cK$ where $\tilde B$ is some ball of radius $1$. By convexity, we have (recall also that $\mathrm{diam}(\cK) \leq T$)
$$
\left ( 1 - \frac{1}{T^3}\right ) y_t + \frac{1}{T^3} \tilde B \subset B \cap \cK
$$ 
which implies that 
$$
\Vol(F_t) \geq T^{-3 n} \omega_n.
$$
The second inequality follows from assumption \eqref{eq:paramassump} and from the well-known inequality $\omega_n \geq \frac{1}{2} n^{-2n}$.
\end{proof}

Remark that by construction, we have that $\Vol(F_{\tau_{i+1}}) / \Vol(F_{\tau_i}) \leq \frac{1}{2}$ for all $i=0,1,..,N-1$ (with $\tau_0 := 1$). Together with the last claim, this yields that
\begin{align}
N ~& \leq - \log_2(\Vol({F_{\tau_N}}) / \Vol(\cK)) \nonumber \\
& \stackrel{ \eqref{eq:volFt} }  {\leq} 4n \log_2(T) + n \log_2 \mathrm{diam} (\cK) \label{eq:boundN} \leq 5 n \log_2 T.
\end{align}
At this point, we will set
\begin{equation} \label{eq:defgamma}
\gamma = \frac{1}{5 n \log_2 T} , 
\end{equation}
which implies that $\eta_{\tau} / \eta_1 \leq e$. \\

Finally, we establish the following lower bound on the covariance of $p_t$:
\begin{claim} \label{claim:detcov}
	We have, for all $t \in [\TF]$,
	\begin{equation} \label{eq:covptlarge}
	\log \det \COV(p_t) \geq - 6 n \log (T).
	\end{equation}
\end{claim}
\begin{proof}
	We have by definition of $\cE_{p}$
	$$
	\Vol(\cE_p(r)) = \omega_n \det \COV(p)^{1/2} r^n
	$$
	where $\omega_n$ denotes the volume of the unit ball in $\RR^n$. Moreover, we have by construction and by the previous claim,
	$$
	\Vol(\cE_{p_t}(\alpha)) \stackrel{ \eqref{eq:shrinkFt} }{\geq} \frac{1}{2} \Vol(F_{t-1}) \stackrel{ \eqref{eq:volFt} }{\geq} \frac{1}{2} T^{-3n} \omega_n.
	$$
	Plugging these two equations together yields
	$$
	\det \COV(p_t)^{1/2} = \frac{\Vol(\cE_{p_t}(\alpha))}{\omega_n \alpha^n} \geq \frac{1}{2 \alpha^n T^{3n}} .
	$$
	Together with equation \eqref{eq:paramassump}, this completes the proof.
\end{proof}

The next claim shows that $\tilde \ell_t$ is regular in $F_t$:
\begin{claim} \label{lem:zeta}
	For all $t \in [\TF]$ and all $y \in F_t$, one has that, almost surely,
	\begin{equation}\label{eq:regularity1}
	\tilde \ell_t(y) \leq \zeta \mbox{ and } \left \| \nabla \tilde \ell_t(y) \right \|_{\COV(p_t)} \leq \zeta'
	\end{equation}
	where
	$$
	\zeta = \exp\left(C \frac{n \log(T)}{\epsilon^2} n \alpha \sqrt{\lambda} \right),
	$$
	$$
	\zeta'= C \frac{n \log(T)}{\epsilon^2} \zeta
	$$
	and $C=10^3$.
\end{claim}
\begin{proof}
	The result is an immediate application of Lemma \ref{lem:smoothnesstildeell} with $R_2 = 10 n \alpha$ and $R_1 = 10 n \alpha \lambda + 20 \sqrt{\lambda} \epsilon$. With the help of equation \eqref{eq:FtInEllipsoid} we have that $F_t \subset \cE_{p_t}(R_2)$ which gives that
	$$
	f_x(y) := \frac{K[p_t] (x,y)}{K[p] p_t(x)} \leq \zeta, ~~ \forall x \in \Omega_t, ~ \forall y \in F_t
	$$
	and that $|\nabla f_x(y)| \leq \zeta'$ for all $y \in \mathrm{int}(F_t)$. The result now immediately follows by definition of $\tilde \ell_t$, the fact that $\ell_t(x_t) \in [0,1]$ almost surely and the bounds $\lambda,\epsilon \leq 1/2$ and $n \alpha \sqrt{\lambda} \leq 1$.
\end{proof}

We take $\lambda$ to be small enough so that $\zeta \leq e$. That is, we set
\begin{equation} \label{eq:deflambda}
\lambda = \frac{\eps^4}{C^2 n^4 \alpha^2 \log^2(T)}.
\end{equation}
where $C$ is the constant from the above lemma. With these choices and with the assumption \eqref{eq:paramassump} we conclude that
\begin{equation}\label{zetabound}
\zeta \leq e, ~~ \zeta' \leq C n  \log(T) \zeta / \epsilon^2 \leq T^2.
\end{equation}
where, in the above, we used the assumption that $T$ is larger than some universal constant. 

\subsection{Concentration} \label{sec:conc}
Our goal in this section is to prove Proposition \ref{prop:conc}. We set
\begin{equation} \label{eq:defeta}
\eta_1 = \frac{1}{20 e^2 \sqrt{n T \log\left(T\right) }}
\end{equation}
which gives that
$${\tiny{ B_t = \left  \{ \max\left( \Delta_t^{(1)}, \dots, \Delta_t^{(4)} \right) \leq \eta_1 20 e^2 \sqrt{n T \log T} \right\}.}}$$

We begin with two simple estimates concerning large deviations of $K[p_t]p_t$.
\begin{lemma} \label{lem:seb6}
For all $t \leq \TF$, one has that 
\begin{equation}\label{eq:OmegaInK}
\cE_{p_t}(10 n \alpha \lambda + 20 \sqrt{\lambda} \epsilon) \subset \cK
\end{equation}
and
\begin{equation}\label{eq:xtInOmega}
K[p_t]p_t(\Omega_t) \geq 1 - 1/T^2.
\end{equation}
\end{lemma}
\begin{proof}
Equation \eqref{eq:OmegaInK} is a direct consequence of Lemma \ref{lem:EllipsoidInSupport} combined with the fact that $10 n \alpha \lambda + 20 \sqrt{\lambda} \epsilon \leq 1/100$ (recall the value of $\lambda$ given by \eqref{eq:deflambda}). In other words, we have that $\Omega_t =  \cE_{p_t}(10 n \alpha \lambda + 20 \sqrt{\lambda} \epsilon)$. Now, according to equation \eqref{eq:FtInEllipsoid} we have $Y \in \cE_{p_t}(10 n \alpha)$ almost surely when $Y \sim p_t$. Thus, we can write
\begin{align*}
\P_{X \sim K[p_t]p_t} (X \not\in \Omega_t) ~& = \P_{C \sim c[p_t], Y \sim p_t } ((1-\lambda) C + \lambda Y \notin  \cE_{p_t}(10 n \alpha \lambda + 20 \sqrt{\lambda} \epsilon) ) \\
&\leq \P_{C \sim c[p_t]} (C \not\in \cE_{p_t}(20 \sqrt{\lambda} \epsilon)) \\
&\leq \P_{X \sim \cN \left (0, \frac{1}{20 n \log(T)} \mI_n \right )} (|X| \geq 1) \leq \frac{1}{T^2}
\end{align*}
where the last inequality follows for example as an application of Lemma \ref{lem:GaussianTail}. The proof is complete.
\end{proof}

We also need the following bound:
\begin{lemma} \label{lem:extapprox}
For every $t \leq \TF$ and for any $y \in F_t$, one has that 
\begin{equation}\label{eq:extapprox}
|K[p_t]^* \underline{\ell}_t(y) - K[p_t]^* \ell^{ext}_t(y)| \leq 1/T^2.
\end{equation}
\end{lemma}
\begin{proof}
Since $\underline{\ell}_t(x) = \ell_t^{ext}(x)$ for all $x \in \Omega_t$, and since both functions are $T$-Lipschitz on the interior of $\Omega_t^C$, it follows that
\begin{equation}\label{eq:ellell}
|\underline{\ell}_t(x) - \ell_t^{ext}(x)| \leq T d(x, \Omega_t) + \ds1\{x \notin \Omega_t \}, ~~ \forall x \in \RR^n.
\end{equation}
We thus have
\begin{align*}
|K[p_t]^* \underline{\ell}_t(y) - ~& K[p_t]^* \ell^{ext}_t(y)|  = \left | \E_{X \sim (1 - \lambda) c[p_t] + \lambda y} (\underline{\ell}_t(X) - \ell^{ext}_t(X) ) \right | \\
& \leq \E_{X \sim (1 - \lambda) c[p_t] + \lambda y} (|\underline{\ell}_t(X) - \ell^{ext}_t(X)| ) \\
& \stackrel{ \eqref{eq:ellell} }{\leq}  \E_{X \sim (1-\lambda)c[p_t] } \left (T d(X+ \lambda y, \Omega_t) + \ds1\{X + \lambda y \notin \Omega_t\} \right ) \\
& \stackrel{ \eqref{eq:FtInEllipsoid}, \eqref{eq:OmegaInK}}{\leq} \E_{X \sim c[p_t]} \left ( T d(X, \cE_{p_t}(20 \sqrt{\lambda} \epsilon)) + \ds1\{X  \notin \cE_{p_t}(20 \sqrt{\lambda} \epsilon) \right ) \\
& \leq T \E_{X \sim \cN \left (0, \frac{1}{20 n \log(T)} \mI_n \right )} \left (|X| + 1 \right ) \ds1\{|X|>1\}  \leq \frac{1}{T^2}.
\end{align*}	
where the last inequality is an application of Lemma \ref{lem:GaussianTail}.
\end{proof}

Consider the filtration $\cF_t = \sigma(\ell_1, x_1, \ell_2, x_2, \hdots, \ell_t, x_{t}, \ell_{t+1})$. We define the random variables
$$
U_t(y) = \begin{cases} \tilde{\ell}_t(y) - K[p_t]^*\underline{\ell}_t  (y) & t \leq \TF \mbox{ and } y \in F_t \\ 0 & \text{otherwise}
\end{cases} ,~~ V_t(y) = \frac{\eta_t}{\eta_1} U_t(y), ~~ \forall y \in \cK
$$
and moreover we set
$$
W_t := \begin{cases} \langle p_t, \tilde{\ell}_t - K[p_t]^* \underline{\ell}_t \rangle  & t \leq \TF  \\ 0 & \text{otherwise}
\end{cases}
$$
and
$$
S_t := \begin{cases} \langle p_t, K[p_t]^* \underline{\ell}_t \rangle - \underline{\ell}_t(x_t)  & t \leq \TF  \\ 0 & \text{otherwise}
\end{cases}.
$$

We claim that these four functions are martingale differences with respect to the filtration $\cF_t$:
\begin{claim}
For all $t \geq 1$ and all $y \in \cK$, we have almost surely that
\begin{equation}\label{eq:martingale}
\E[W_t | \cF_{t-1}] = \E[S_t | \cF_{t-1}] = \E[U_t(y) | \cF_{t-1}] = \E[V_t(y) | \cF_{t-1}] = 0.
\end{equation}
\end{claim}
\begin{proof}
A key observation towards proving the claim is that, for all $y \in F_t$,
\begin{align}
\E \left [\tilde \ell_t(y) | \cF_{t-1} \right ] ~& = \E_{X \sim K[p_t] p_t} \left [\frac{\underline{\ell}_t(X)}{K[p_t] p_t (X)} K[p_t](X,y) \right ] \label{eq:exptildeell} \\
& = \int_{\RR^n} \underline \ell_t(x) K[p_t](x,y) dx = K[p_t]^* \underline{\ell}_t(y). \nonumber
\end{align}

This immediately shows that, for every $t \geq 1$ and $y \in \cK$, $\E(U_t(y) | \cF_{t-1}) = 0$, and the same is true for $V_t$. Moreover, since $p_t$ is measurable with respect to $\cF_{t-1}$, which gives, using Fubini's theorem, that $\E[W_t | \cF_{t-1}] = 0$. Finally, by the definition of $x_t$, we have that $\E[S_t| \cF_{t-1}] = 0$. This completes the claim.
\end{proof}

We will use the Azuma-Hoeffding inequality:
\begin{theorem}(Azuma-Hoeffding)
	Let $c>0$ and let $M_1,M_2,..$ be a martingale satisfying $|M_{t+1} - M_{t}| < c$ almost surely for all $t \geq 1$. Then
	\begin{equation}\label{azuma}
	\P( |M_t - M_1| \geq u ) \leq 2 \exp \left ( - \frac{ u^2}{ 2 c^2 t} \right ), ~~ \forall u > 0.
	\end{equation}
\end{theorem}

We would like to apply the above bound for the martingales $\sum_{s=1}^t U_s(y), \sum_{s=1}^t V_s(y), \sum_{s=1}^t W_s$ and $\sum_{s=1}^t S_s$, which requires us to first prove an almost-sure bound for the respective martingale differences. To that end, we recall equation \eqref{eq:regularity1} and \eqref{zetabound} which ensure that, almost surely, 
$$
|U_t(y)| \leq \zeta + 1 \leq 2 e, ~~ \forall t \geq 1, \forall y \in \cK.
$$
The same argument also ensures that $|V_t(y)| \leq e (\zeta+1) \leq 2 e^2$ since as we observed in Section \ref{sec:usefulobs} one has $\eta_{\tau} / \eta_1 \leq e$. Moreover, since by definition one has that $\left | \langle K[p_t] p_t, \underline {\ell}_t \rangle - \underline{\ell}_t(x_t) \right | \leq 2$, we also have $|S_t| \leq 2$. Finally the inequality $|U_t(y)| \leq 2 e$ implies that $|W_t| = |\langle p_t, U_t \rangle| \leq 2 e$. We conclude that
$$
\max \left ( |U_t(y)|, |V_t(y)|, |W_t|, |S_t| \right ) \leq 2 e^2, ~~ \forall t \geq 1, \forall y \in \cK.
$$

Using equation \eqref{azuma} and a union bound, we get that for any $t \geq 1$, for all $y \in \cK$ and for all $\delta > 0$, with probability at least $1-\delta$,
\begin{equation}\label{eq:azuma1}
\max\left(\left| \sum_{s=1}^t U_s(y) \right|, \left| \sum_{s=1}^t V_s(y) \right|, \left| \sum_{s=1}^t W_s \right|, \left| \sum_{s=1}^t S_s \right| \right) \leq 2 e^2 \sqrt{2T \log\left( \frac{8}{\delta} \right) }.
\end{equation}
We want this to hold simultaneously for all $y \in F_t$, this is where our estimates on the Lipschitz constant will come to play. We will need the following lemma.

\begin{lemma} \label{lem:simultanously}
Let $\cK \subset \RR^n$ be a convex domain and $\delta,v,M,L > 0$. Let $F \subset \cK$ be a random convex subset of $\cK$, $C$ a random matrix and $f: \cK \to [0,\infty)$ be a random function, which satisfy the following conditions:
\begin{enumerate}[label=(\roman*)]
\item
$\Vol(F) \geq v \Vol(\cK)$ almost surely. 
\item
For all $x \in \cK$,
$$
\P( f(x) \geq M) \leq \delta.
$$
\item \label{item:constainedinellipsoid}
Almost surely, for all $x,y \in F$ one has that $\|x-y\|_{C^{-1}} \leq 1$.
\item \label{item:lipschitz}
Almost surely we have
$$ 
\| \nabla f(x) \|_{C} \leq L,~~  \forall x \in \mathrm{int}(F).
$$
\end{enumerate}
Then,
$$
\P( f(x) \leq 2 M, ~~ \forall x \in F ) \geq 1 - \frac{\delta}{v} \left ( \frac{L}{M} \right )^n.
$$
\end{lemma}

The proof is postponed to the end of the section. We are now ready to prove Proposition \ref{prop:conc}.
\begin{proof} [Proof of Proposition \ref{prop:conc}]
Define,
$$
f_t(x) = \max\left(\left| \sum_{s=1}^t S_s \right|, \left| \sum_{s=1}^t W_s \right|, \left| \sum_{s=1}^t U_s(x) \right|, \left| \sum_{s=1}^t V_s(x) \right|\right), \forall t \geq 1, ~ x \in \cK.
$$
We first claim that
\begin{equation}\label{eq:TFsmall}
\TF < \tau \Rightarrow \exists y \in F_{\TF} \mbox{ such that }  f_{\TF} (y) \geq 2M \mbox { or } x_{\TF} \notin \Omega_{\TF}
\end{equation}
where
$$
M = 10 e^2 \sqrt{n T \log T}.
$$
Indeed, suppose that the event $\TF < \tau$ holds. Using Lemma \ref{lem:extapprox}, we have that for all $y \in F_{\TF}$,
$$
\left | \tilde{L}_{\TF}(y) - L_{\TF}(y) \right | = \left |\sum_{s=1}^{\TF} \left ( K[p_s]^* \underline{\ell}_s(y) - K[p_s]^* \ell_s^{ext}(y) + U_s(y) \right ) \right | \stackrel{\eqref{eq:extapprox}}{\leq} f_t(y) + \frac{1}{T}
$$
or in other words $\Delta_{\TF}^{(2)} \leq \eta_1 \max_{y \in F_{\TF}} f_t(y) + \frac{1}{T}$. Following the same argument, we also have that
$$
\max \left ( \Delta_{\TF}^{(3)}, \Delta_{\TF}^{(4)} \right ) \leq \eta_1 \max_{y \in F_{\TF}} f_t(y) + \frac{1}{T}.
$$
Finally, we also have by definition that 
$$
\Delta_{\TF}^{(1)} \leq \eta_1 \max_{y \in F_{\TF}} f_{\TF}(y) + \ds 1 \{x_{\TF} \notin \Omega_{\TF} \}.
$$
A combination of the last 3 displays finally gives \eqref{eq:TFsmall}.

Therefore, in order to complete the proof we only need to show (thanks to Lemma \ref{lem:seb6}) that 
\begin{equation} \label{eq:nts1}
\P \left ( f_t(y) \leq 2M, ~ \forall t \in [T], \forall y \in F_t \right ) \geq 1 - \frac{1}{T^3}.
\end{equation}
We use equation \eqref{eq:azuma1} with $\delta = T^{-12n}$ to get that for all $x \in \cK$,~ $\P(f_t(x) > M) \leq \delta$. 
Next, define $v = T^{-4n}$ and $C = (10 n \alpha)^2 \COV(p_t)$. We have according to Claim \ref{claim:volFt} that $\Vol(F_t) \geq v \Vol(\cK)$ almost surely, and according to equation \eqref{eq:regularity1} and \eqref{zetabound}, we have that $\| \nabla \tilde \ell_t (x)\|_{C}\leq 10 n \alpha T^2 \leq T^3$ which implies that
$$
\| \nabla f_t (x)\|_{C} \leq T^4.
$$
Moreover, according to equation \eqref{eq:FtInEllipsoid}, we have $F_t \subset \cE_{p_t}(10 n \alpha)$ and thus for all $x,y \in F_t$ we have $\|x-y\|_{C^{-1}} \leq 1$. According to the above, we may use Lemma \ref{lem:simultanously} to deduce that
$$
\P( f_t(x) \leq 2 M, ~~ \forall x \in F_t ) \geq 1 - \frac{\delta}{v} T^{4n} \geq 1 - T^{-4n} \geq 1 - 1/T^4.
$$	
By using a union bound on $t$ equation \eqref{eq:nts1} follows and the proof is complete.
\end{proof}

It remains to prove Lemma \ref{lem:simultanously}.

\begin{proof}[Proof of Lemma \ref{lem:simultanously}]
Let $E$ be the event that there exists a point $x \in F$ with $f(x) \geq 2 M$. Suppose that the latter event occurs. By convexity, we have that
$$
F' := (1 - \lambda) x + \lambda F \subset F
$$
for $\lambda = \frac{M}{L}$. Now, according to \ref{item:constainedinellipsoid} we have that, for all $y \in F'$, $\|y-x\|_{C} \leq \lambda$. Thus, using \ref{item:lipschitz}, we get that 
\begin{align*}
|f(x) - f(y)| ~& = \left | \int_0^1 \left \langle \nabla f ((1-\theta) x + \theta y), x-y \right \rangle d \theta \right | \\
& \leq \int_0^1 \| \nabla f((1-\theta) x + \theta y))\|_{C^{-1}} \|x-y\|_{C} d \theta \leq \lambda L 
\end{align*}
and therefore
$$
f(y) \geq  M, ~~ \forall y \in F'.
$$
Observing that 
$$
\Vol(F') = \left ( \frac{M}{L} \right )^n \Vol(F) \geq \left ( \frac{M}{L} \right )^n v \Vol(\cK),
$$
we deduce that
$$
E \mbox{ holds } \Rightarrow \frac{1}{\Vol(\cK)} \int_{\cK} \ds1\{f(x)\geq M\} dx \geq \left ( \frac{M}{L} \right )^n v.
$$
On the other hand, by Fubini's theorem,
$$
\E \left [ \frac{1}{\Vol(\cK)} \int_{\cK} \ds1\{ f(x) \geq M \} dx \right ] \leq \delta.
$$
Plugging the last two displays together, we get that
$$
\P(E) \leq \frac{\delta}{v} \left ( \frac{L}{M} \right )^n.
$$
\end{proof}

\subsection{Standard analysis of exponential weights} \label{sec:standardanalysis}
We adapt here the usual analysis of exponential weights to deal with our adaptive learning rate $(\eta_t)$. First we restate the usual bound for time-dependent learning rate.

\begin{lemma} \label{lem:standardanalysis}
Let $\cK \subset \RR^n$ be a compact set with nonempty interior and $\tau \geq 2$. Let $p_1:\cK \to [0,\infty)$ be a probability density on $\cK$, let $f_1, \hdots, f_{\tau}:\cK \to [0,+\infty)$ be measurable functions, let $\eta_1,...,\eta_\tau \in (0,+\infty)$ and let $\cK = F_1 \supset F_2 \supset ... \supset F_\tau$ be a decreasing sequence of subsets of $\cK$ with non-empty interior. By induction construct $p_t$, for all $1 \leq t \leq \tau$, by
$$p_{t+1}(x) = \frac{p_t(x) \exp\left( - \eta_t f_t(x) \right) \ds 1\{x \in F_t \}}{\int_{F_t} p_t(y) \exp\left( - \eta_t f_t(y) \right) dy}  , \; \forall x \in \cK .$$
Then for every $x \in F_{\tau}$ we have
\begin{equation} \label{eq:standardanalysis1}
\sum_{t=1}^{\tau} \langle p_t - \delta_x, f_t \rangle \leq \sum_{t=1}^\tau \frac{\log(p_{t+1}(x)) - \log(p_t(x))}{\eta_t} + \sum_{t=1}^\tau \eta_t \langle p_t, f_t^2 \rangle.
\end{equation}
\end{lemma}
\begin{proof}
An elementary calculation yields for any $x \in \cK$,
$$\langle p_t - \delta_x, f_t \rangle = \frac{\log(p_{t+1}(x)) - \log(p_t(x))}{\eta_t} + \frac{1}{\eta_t} \log \E_{X \sim p_t} \exp \left(- \eta_t (f_t(X) - \E_{X' \sim p_t} f_t(X')) \right) .$$
Using that $f_t(x) \geq 0$ for any $x \in \cK$ and $t \in [\tau]$, and that $\log(1+s) \leq s$ and $\exp(-s) \leq 1 - s + s^2$ for any $s \geq 0$ one has 
$$\log \E_{X \sim p_t} \exp \left(- \eta_t (f_t(X) - \E_{X' \sim p_t} f_t(X')) \right)  \leq \eta_t^2 \E_{X \sim p_t} f_t(X)^2.$$
Plugging the last two displays together concludes the proof of \eqref{eq:standardanalysis1}.
\end{proof}

Using the bounds \eqref{eq:regularity1} and \eqref{zetabound}, we have 
\begin{equation}\label{key}
\sum_{t=1}^{\TF} \eta_t \langle p_t, \tilde{\ell}_t^2 \rangle \leq e^3 \eta_1 \TF.
\end{equation}

Let $\tau_1, \hdots, \tau_{N'}$ the times in $\{1,\hdots, \TF-1\}$ at which we increase the learning rate (since $\TF \leq \tau$, we have $N' \leq N$), that is $\eta_{\tau_i+1} = (1+\gamma) \eta_{\tau_i}$. We observe that for all $x \in F_{\TF}$, one has 
\begin{align*}
	\sum_{t=1}^{\TF} & \frac{\log p_{t+1}(x) - \log p_t(x)}{\eta_t} \\ ~& = \sum_{t=1}^{\TF} \left ( \log p_{t+1}(x) - \log p_t(x) \right ) \left ( \frac{1}{\eta_{\TF}} + \sum_{s=t}^{{\TF-1}} \left ( \frac{1}{\eta_s} - \frac{1}{\eta_{s+1}} \right ) \right ) \\ 
	& = \frac{1}{\eta_{\TF}} \left ( \log p_{{\TF+1}}(x) - \log p_1(x)) \right ) + \sum_{s=1}^{{\TF-1}} \left ( \frac{1}{\eta_s} - \frac{1}{\eta_{s+1}} \right ) \sum_{t=1}^s \left ( \log p_{t+1}(x) - \log p_t(x) \right ) \\
	& = \frac{1}{\eta_{\TF}} \left ( \log p_{{\TF+1}}(x) - \log p_1(x) \right ) + \frac{\gamma}{1 + \gamma}\sum_{i=1}^{N'} \frac{1}{\eta_{\tau_i}} \left ( \log p_{\tau_{i+1}}(x) - \log p_1(x) \right ).
\end{align*}

Combining the last two displays and using \eqref{eq:standardanalysis1} of the previous lemma we finally get for all $x \in F_{\TF}$
\begin{align*}
\sum_{t=1}^{\TF} \langle p_t - \delta_x, \tilde{\ell}_t \rangle ~& \leq \frac{1}{\eta_{\TF}} \left ( -Q_{\TF}(x) + \log(Z_1/Z_{\TF+1}) \right ) + \frac{\gamma}{1 + \gamma}\sum_{i=1}^{N'} \frac{1}{\eta_{\tau_i}} \left ( -Q_{\tau_i}(x) + \log(Z_1/Z_{\tau_{i+1}})  \right ) + e^3 \eta_1 \TF \\
& \leq \frac{1+ \gamma N}{\eta_1} \max_{t \in [\TF+1]} |\log(Z_1/Z_{t})| - \frac{\gamma}{2 e \eta_1} \max_{i \in [N']} Q_{\tau_i}(x) + e^3 \eta_1 \TF.
\end{align*}
On the other hand, we have for all $t \in [\TF+1]$,
\begin{align*}
|\log(Z_1/Z_{t})| ~& \stackrel{\eqref{eq:logZBound}}{\leq} n (\log n + 13) - \frac{1}{2} \log \det \COV(p_{t}) + \frac{1}{2} \log \det \COV(p_1) \\
& \stackrel{\eqref{eq:covptlarge}}{\leq} n (\log n + 3 + 4 \log T) \leq 6 n \log T.
\end{align*}
The above two equations together with \eqref{eq:boundN}, \eqref{eq:defgamma} and \eqref{eq:defeta} finally yield that for all $x \in F_{\TF}$,
$$
\sum_{t=1}^{\TF} \langle p_t - \delta_x, \tilde{\ell}_t \rangle \leq \sqrt{T} \left ( 2^{13} (n \log T)^{3/2} - \frac{1}{\sqrt{n \log T}} \max_{i \in N} Q_{\tau_i}(x) \right )
$$
which implies, in particular,
\begin{align}\label{eq:finalregretest}
\sum_{t=1}^{\TF} \langle p_t, \tilde{\ell}_t \rangle - ~& \min_{x \in F_{\TF}} \tilde L_{\TF}(x) \\
& \leq \sqrt{T} \left ( 2^{14} (n \log T)^{3/2} - \frac{1}{2 \sqrt{n \log T}} \max_{y \in E_{\TF}} \max_{i \in [N']} Q_{\tau_i}(y) \right ) \nonumber
\end{align}
where
$$
E_{t} := \left \{y \in \cK; ~ \eta_1 \left (\tilde{L}_t(y) - \min_{x \in F_{t}} \tilde L_{t}(x) \right ) \leq \beta \right  \}.
$$

\subsection{Final analysis} \label{sec:finalanalysis}
In this section we finally prove Proposition \ref{prop:mainstep}. We begin with the following proposition, which extracts the main idea of using a kernel (this calculation is similar to what we did in Theorem \ref{th:basic}).
\begin{proposition}
\begin{equation}\label{eq:kernel1}
\max_{x \in \cK}  \sum_{t=1}^{\TF} (\langle p_t, K[p_t]^* \underline{\ell}_t \rangle - \ell_t(x)) \leq 1 +\frac{1}{\lambda} \left (\sum_{t=1}^{\TF} \langle p_t, K[p_t]^* \ell_t^{ext} \rangle - \min_{x \in \cK} [ L_{\TF}(x) ] \right )
\end{equation}
\end{proposition}
\begin{proof}
By definition, for all $t \leq \TF$ we have that $p_t$ is $(1/e)$-approximately log-concave, and thus Lemma \ref{lem:convexdommain} teaches us that
\begin{equation}\label{eq:convexity1}
\langle c[p_t], \ell_t^{ext} \rangle \leq \langle K[p_t]p_t, \ell_t^{ext} \rangle + \frac{1}{T^2} .
\end{equation}
In particular, by convexity of $\ell_t^{ext}$, we have for all $x \in \cK$,
\begin{align*}
\langle K[p_t] \delta_x , \ell_t^{ext} \rangle ~& = \E_{C \sim c[p_t]} \ell_t^{ext} (\lambda x + (1- \lambda) C) \\  
& \leq \lambda \ell^{ext}_t(x) + (1-\lambda) \langle c[p_t], \ell_t^{ext} \rangle \\
& \stackrel{\eqref{eq:convexity1}}{\leq} \frac{1}{T^2} +  \lambda \ell^{ext}_t(x) + (1-\lambda) \langle K[p_t]p_t, \ell_t^{ext} \rangle
\end{align*}
which in turn gives
$$\langle K[p_t] p_t - \delta_x , \ell_t^{ext} \rangle \leq \frac{1}{\lambda T^2} + \frac{1}{\lambda} \langle p_t - \delta_x , K[p_t]^* \ell_t^{ext} \rangle.$$
Since $\underline{\ell}_t \leq \ell_t^{ext}$, we get that for all $x \in \cK$,
$$
\sum_{t=1}^{\TF} (\langle p_t, K[p_t]^* \underline{\ell}_t \rangle - \ell_t(x)) \leq 1 +\frac{1}{\lambda} \left (\sum_{t=1}^{\TF} \langle p_t, K[p_t]^* \ell_t^{ext} \rangle - L_{\TF}(x) \right ).
$$
This completes the proof.
\end{proof}

We aim to use the estimate \eqref{eq:finalregretest} of the previous section in order to bound from above the right hand side of \eqref{eq:kernel1}. First we show that those estimates yield an upper bound on the regret. To that end, we need to use bounds that connect the functions $K[p_t]^* \ell_t^{ext}$ and $\tilde \ell_t$. By definition of $\TF$, we have that $\Delta_{\TF-1}^{(1)}, \Delta_{\TF-1}^{(4)} \leq 1$, which teaches us that
$$
\sum_{t=1}^{\TF}  \ell_t(x_t)  \leq \sum_{t=1}^{\TF} \langle p_t, K[p_t]^* \underline{\ell}_t \rangle  + \frac{2}{\eta_1}
$$
and also, 
$$
\sum_{t=1}^{\TF} \langle p_t, K[p_t]^* \ell_t^{ext} \rangle \leq \sum_{t=1}^{\TF} \langle p_t, \tilde \ell_t \rangle + \frac{2}{\eta_1}.
$$
Combining the two above displays with equation \eqref{eq:kernel1} gives
\begin{equation}\label{eq:kernel2}
\max_{x \in \cK}  \sum_{t=1}^{\TF} (\ell_t(x_t) - \ell_t(x)) \leq 1 +\frac{1}{\lambda} \left (\sum_{t=1}^{\TF} \langle p_t, \tilde \ell_t \rangle - \min_{x \in \cK} [ L_{\TF}(x) ] + \frac{4}{\eta_1} \right )
\end{equation}

Next, we would like to bound the term $\min_{x \in \cK} L_{\TF}(x)$ from below. To that end, we will need to show that the minimizer of $L_t$ is attained inside $F_t$, and can therefore be approximated via $\tilde L_t$. This follows from the restart condition, as demonstrated by the next lemma.
\begin{lemma}
For all $t < \TF$,
\begin{equation}\label{eq:argminL}
\arg \min_{x \in \cK} L_{t}(x) \in F_{t+1}.
\end{equation}
\end{lemma}
\begin{proof}
By definition of $\tau$, the fact that $t<\tau$ (since, by definition, $\TF \leq \tau$) implies
$$
\eta_1 (\tilde{L}_t(x) - \min_{y \in F_t} \tilde{L}_t(y)) \geq \beta = 4, ~~ \forall x \in  \partial F_{t+1} \cap \mathrm{int}(\cK).
$$
On the other hand, the fact that $t \leq \TF$ implies that
$$
\eta_1 \left| \tilde{L}_t(x) - L_t(x) \right| < 1, \forall x \in F_{t+1}.
$$
Combining those two inequalities teaches us that
$$
\min_{x \in \partial F_{t+1} \cap \mathrm{int}(\cK)} L_t(x) > \min_{x \in F_{t+1}} L_t(x).
$$
It follows by convexity that $\arg \min_{x \in \cK} L_{t}(x) \in F_{t+1}$. The proof is complete.
\end{proof}

Applying the above lemma with $t = \TF-1$ and using the definition of $\TF$ and the fact that $L_t(x) - L_{t-1}(x) \leq 1$ for all $x \in \cK$ and all $t < \TF$, we get that
\begin{align}\label{eq:minmin}
- \min_{x \in \cK} L_{\TF}(x) ~& \leq - \min_{x \in \cK} L_{\TF-1}(x) + 1 \nonumber \\
& \stackrel{ \eqref{eq:argminL} }{\leq} - \min_{x \in F_{\TF}} L_{\TF-1}(x) + 1 \nonumber \\
& \leq - \min_{x \in F_{\TF}} \tilde L_{\TF-1}(x) + \max_{y \in F_{\TF}} \left |L_{\TF-1}(y) - \tilde L_{\TF-1}(y) \right | + 1  \nonumber \\
& \leq - \min_{x \in F_{\TF}} \tilde L_{\TF} (x) + \frac{2}{\eta_1}.
\end{align}

Combining the last display with equations \eqref{eq:finalregretest} and \eqref{eq:kernel2}, we finally get
\begin{align} \label{eq:finalest}
\max_{x \in \cK}  \sum_{t=1}^{\TF} (\ell_t(x_t) - \ell_t(x)) ~& \stackrel{ \eqref{eq:kernel2} }{\leq} 1 +\frac{1}{\lambda} \left (\sum_{t=1}^{\TF} \langle p_t, \tilde \ell_t \rangle - \min_{x \in \cK} L_{\TF}(x) + \frac{4}{\eta_1} \right ) \nonumber \\
& \stackrel{\eqref{eq:minmin} }{\leq} 1 +\frac{1}{\lambda} \left (\sum_{t=1}^{\TF} \langle p_t, \tilde \ell_t \rangle - \min_{x \in F_{\TF} } \tilde L_{\TF}(x) + \frac{6}{\eta_1}  \right )  \nonumber\\
& \stackrel{\eqref{eq:finalregretest}}{\leq} \frac{\sqrt{T}}{\lambda} \left (  2^{15} (n \log T)^{3/2} - \frac{1}{\sqrt{n \log T}} \mathcal{Q}  \right ). 
\end{align} 
where $\mathcal{Q} = \max_{\tau_i \leq \TF} \max_{y \in E_{\TF}} Q_{\tau_i}(y)$. \\

Finally, we need the following claim in order to finish the proof of Proposition \ref{prop:mainstep}. For the sake of this claim, we choose
\begin{equation} \label{eq:defalpha}
\alpha = (2e)^{17} n^2 \log(T)^2 .
\end{equation}
\begin{claim} \label{claim:Qbig}
Under the event $\TF = \tau < T$, we have almost surely that $\mathcal{Q} \geq  2^{16} (n \log T )^2$.
\end{claim}
\begin{proof}
The event $\TF = \tau < T$ means that the restart condition holds true that time $\tau$. Let $z \in \partial F_{\tau+1} \cap \mathrm{int}(\cK)$ be the point that triggered the restart. By definition, we have that $z \in E_{\TF}$, which implies that $\mathcal{Q} \geq \max_{\tau_i \leq \TF} Q_{\tau_i}(z)$.
Let $i \in [N]$ be the largest integer for which $z \in \mathrm{int}(F_{\tau_i})$. Since we have that $z \notin \mathrm{int}(F_{\tau_i+1})$, by construction of $F_t$ we have that 
$$
F_{\tau_i + 1} = F_{\tau_i} \cap \cE_{p_{\tau_i+1}}(\alpha)
$$
which implies that
$$
z \notin \mathrm{int} (\cE_{p_{\tau_i+1}}(\alpha)).
$$ 
An application of Lemma \ref{lem:covtovalue} with $p=p_{\tau_{i+1}}$ and $\epsilon = 1/e$ now gives that
$$\alpha \leq \|z-\mu(p)\|_{\mathrm{Cov}(p)^{-1}} \leq \exp(17) \left( Q_{\tau_i}(z) + 10 n \log n) \right).$$ 
The choice of $\alpha$ in \eqref{eq:defalpha} gives
$$
\mathcal{Q} \geq Q_{\tau_i}(z) \geq \alpha e^{-17} - 10 n \log n \geq 2^{16} (n \log T )^2.
$$
\end{proof}

Finally we obtain:
\begin{proof}[Proof of Proposition \ref{prop:mainstep}]
Combine equation \eqref{eq:finalest} with claim \ref{claim:Qbig} and equation \eqref{eq:deflambda}.
\end{proof}

\section{Implementation} \label{sec:time}
In this section, we discuss the modification of Algorithm \ref{fig:alg} needed to obtain a polynomial time algorithm that achieves $O(n^{10.5} \log^{7.5}(T) \sqrt{T})$ regret. For simplicity, we assume that $\cK$ is a polytope defined with polynomially many linear constraints. The main difficulties for implementing Algorithm \ref{fig:alg} are to sample from the exponential weights strategy $p_t$, to compute the kernel, and to test the restart condition.

\subsection{Sampling from $p$ in $\mathrm{poly}(n, \log(T)) T$-time} \label{sec:samplingpt}
Fix $t \in [T]$, suppose that the points $x_1,\dots,x_{t-1}$ and the values $\ell_1(x_1), \dots, \ell_{t-1}(x_{t-1})$ have already been determined and that $t \leq \TF$ (note that the stopping times $\tau$ and $\TF$ are measurable with respect to the above). Our objective here is to show how one can efficiently generate a point from the distribution $p=p_t$.

To sample from $p$, we recall that under the above assumptions, it is a $(1/e)$-approximately log-concave function. Our sampling will be based on the following result, which ensures that we can sample a point from $p$ by computing $p(x)/p(y)$ for polynomially many pairs of points.
\begin{theorem} [\cite{BLNR15}] \label{thm:sampling}
Let $\Omega \subset \RR^n$ be a convex set and let $g$ be an $O(1)$-approximately log-concave probability density on $\Omega$. Assume that $\Omega$ contains a unit ball and has diameter at most $D$. We also assume that $|\log(g(x))| \leq M$ for all $x \in \Omega$. Then, we can sample a point according to a probability density function $h$ such that $d_{tv}(g, h)\leq \gamma$ in time 
$$O(\mathrm{poly}(n \log(MD/\gamma)) \mathrm{Oracle})$$
where $\mathrm{Oracle}$ is the maximum between the time needed to compute $g(x)/g(y)$ for any two points 
$x,y\in\Omega$ and the time needed to check if a point is in $\Omega$ or not.
\end{theorem}

To apply this result for the distribution $p_t$, we need to check each parameter:
\begin{enumerate}
\item $F_t$ is contained in the ellipsoid $\cE_{p_t} (20 n \alpha) \subset \cE_{p_t} (\mathrm{poly}(n \log(T)))$ (Claim \ref{claim:log_concave}). Furthermore, $p_t$ contains the ellipsoid $\cE_{p_t}(1/100)$ (Lemma \ref{lem:EllipsoidInSupport}).
From the bound on $|\tilde{\ell}|$ given by Claim \ref{lem:zeta}, it is clear that $1/2\  \cE_{p_t} \subset \cE_{p_{t-1}} \subset 2\cE_{p_t}$. Hence, by a change of variables according to $\cE_{p_{t-1}}$, we have that $D = \mathrm{poly} (n \log(T))$.
\item By the previous item combined with the bound on $\| \nabla \tilde \ell_t \|_{\COV(p_t)}$ given by Claim \ref{lem:zeta}, it is clear that $| \log(p_{t}(x))| \leq \mathrm{poly}(T)$. 
\item Since $F_t$ is the intersection of $O(n \log(T))$-many ellipsoids and $\cK$, we can test if a point is in $F_t$ in time $\mathrm{poly}(n \log(T))$. Since $p_t(x)$ is of the form $\exp(- \sum \eta_i \tilde{\ell}_i(x))/Z$ with $\tilde{\ell}_i(x)$ being Gaussian functions, we can compute $p_{t}(x)/p_{t}(y)$ in time $\mathrm{poly}(n) T$ (here we are glossing over the issue of computing the normalization constant $K[p_t]p_t(x_t)$, see the end of Section \ref{sec:modifiedkernel} for more on this). Therefore, $\mathrm{Oracle} = \mathrm{poly}(n \log(T)) T$. 
\end{enumerate}

Choosing $\gamma = 1/\mathrm{poly}(T)$, this gives the following intermediate result.
\begin{theorem} 
	For every fixed $\kappa > 0$ the following holds. For every $t \in [T]$, given the points $x_1,\dots,x_{t-1}$ and the values $\ell_1(x_1), \dots, \ell_{t-1}(x_{t-1})$ and assuming that $t \leq \TF$, given access to random bits, there is an algorithm that produces a random point $Y \in \cK$ whose distribution has total variation distance from $p_t$ bounded by $1/T^{\kappa}$ and runs in at most $\mathrm{poly}(n \log(T)) T$ time.
\end{theorem}

\subsection{A slightly modified kernel} \label{sec:modifiedkernel}
Our next order of business is to be able to efficiently sample from the distribution $K[p_t] p_t$. To this end, we need to have a rather accurate approximation of $\mu[p_t]$ and $\COV(p_t)$, under which the result of Lemma \ref{lem:convexdommain} will still hold true. A naive approach will be to estimate those parameters by repetitive sampling of $p_t$ and by using sample mean and sample covariance as estimators. Unfortunately, however, in order for our estimator to be accurate enough this would require us to generate some $\mathrm{poly}(T)$ independent samples in each round because in order for the convex domination to hold true, one needs the centroid of $c[p_t]$ to be very close to the centroid of $K[p_t] p_t$.

In order to avoid this issue, we will slightly change the definition of the core $C \sim c[p_t]$. Roughly speaking, instead of a Gaussian whose centroid is $\mu[p_t]$, we will define $C$ as a mixture of translations of such a Gaussian, such that the centroid of the mixture is exactly equal to $\mu[p_t]$, but on the other hand one does not need to know the value of $\mu[p_t]$ in order to sample from $C$.

In order to do this, we will define
$$
\tilde K[p] q \overset{(D)}{=}  (1-\lambda) \tilde c[p]  + \lambda q
$$
where we set
$$
\tilde c[p] = \frac{X_1 + \cdots + X_k}{k} + \cN \left (0, \frac{\epsilon^2}{n \log(T)} \frac{\lambda}{2-\lambda} A \mathrm{Cov}(p) \right )
$$ 
where $X_1,...,X_k$ are independent random variables whose law is $p$, $k$ is an integer to be chosen later and $A$ is a matrix satisfying $\tfrac{1}{2} \mathrm{Id} \preceq A \preceq 2 \mathrm{Id}$ (note that, unlike the definition in Section \ref{sec:highdimkernel}, the centroid of the Gaussian is set at zero).
We claim that for a large enough choice $k = \mathrm{poly}(n \log(T))$, equation \eqref{eq:convexdom} still holds true with the new definitions $\tilde c[p]$ and $\tilde K[p]$. To see this, first observe that since the centroids of $\tilde K[p] p$ and of $\tilde c[p]$ are the same, we may assume that $\mu[p] = 0$. Thus, following the same lines as the proof of the lemma, it is enough to derive the analogue of \eqref{eq:risdominated}, namely to show that
\begin{equation}\label{nts22}
\langle \tilde r[p], g\rangle \leq \langle p, g \rangle + \frac{1}{T^2} 
\end{equation}
where
$$
\tilde r[p] = \cN \left (0, \frac{\epsilon^2 }{n \log(T)} \mathrm{Cov}(p) \right ) + \frac{X_1 + \cdots + X_k}{k} 
$$
and $g$ is a non-negative convex and $T$-Lipschitz function. In view of Lemma \ref{lem:convexdom}, it is enough to show that for a large enough choice $k = \mathrm{poly}(n \log(T))$, one has that 
$$
\P \left ( \frac{X_1 + \cdots + X_k}{k} \notin \cE_{p}(1/(80e)) \right ) \leq \frac{1}{T^2}
$$
which follows by standard concentration estimates and with the help of \eqref{eq:FtInEllipsoid}, which ensures that $\mathrm{supp}(p) \subset  \cE_p(20 \alpha)$.


In order to apply the result from the previous subsection we need to explain how to compute $\tilde{\ell}_t(x)$ in $\mathrm{poly}(n \log(T))$-time with this new kernel. A naive approach would be to use repeated sampling of $\overline{X}= \frac{X_1+\hdots+X_k}{k}$ to estimate $\tilde{K}[p_t](x_t, x)$, however this would again lead to a $\mathrm{poly}(T)$-time computation. Thus we propose to modify the loss estimator by using a single sample of $\overline{X}$. It is clear that this new loss estimator remains unbiased. On the other hand to justify that one still has the same regret guarantee we need to show that Claim \ref{lem:zeta} holds true with our new construction. This follows from the fact that magnitude of the translation by $\overline{X}$ is of order $1/\mathrm{poly}(n \log(T))$ with very high probability, and thus one can easily generalize the proof of Claim \ref{lem:zeta} to this new construction. Finally it remains to explain how to compute the normalization constant $1/\tilde{K}[p_t]p_t(x_t)$. Using the fact that $\tilde{K}[p_t]p_t$ is a mixture of Gaussian whose densities are multiplicatively close to a computable constant $u$ one can reduce the problem to finding an unbiased estimator for each term in the Taylor expansion of $1/(u + \tilde{K}[p_t]p_t(x_t) - u)$. This can again be done via sampling, finally leading to an unbiased and constant-multiplicative approximation of $1/\tilde{K}[p_t]p_t(x_t)$.

\subsection{Generating $x_t$ and checking whether $A_t$ holds} \label{sec:generatext}
Sampling from $\tilde K[p_t] p_t$ amounts to producing $k$ independent samples from $p_t$, which was already settled by the previous subsections, and having a good enough estimate of $\COV(p_t)$. 
By [Corollary 5.52, \cite{Ver12}] (together with standard concentration of log-concave vectors), we know that it takes $(n \log(1/\delta)/\gamma)^{O(1)}$ samples to get a matrix $A$ such that $(1-\gamma) A \preceq \mathrm{Cov}(p_t) \preceq (1+\gamma) A$ with probability at least $1-\delta$. Hence, we only need $\mathrm{poly}(n \log(T))$ random samples of $p_t$ and that takes again time $\mathrm{poly}(n \log(T)) T$. \\

We summarize with the following theorem.
\begin{theorem} 
For every fixed $\kappa > 0$ the following holds. For every $t \in [T]$, given the points $x_1,\dots,x_{t-1}$ and the values $\ell_1(x_1), \dots, \ell_{t-1}(x_{t-1})$ and assuming that $t \leq \TF$, given access to random bits, there is an algorithm that produces a random point $x_t \in \cK$ whose distribution has total variation distance from $\tilde K[p_t] p_t$ bounded by $1/T^{\kappa}$ and runs in at most $\mathrm{poly}(n \log(T)) T$ time.
\end{theorem}

Finally, in order to determine whether or not one should increase the learning rate and update the focus region $F_t$, we need to calculate the ratio $\frac{\mathrm{Vol}(F_t \cap \cE_{p_{t+1}}(\alpha)) }{ \mathrm{Vol}(F_t)}$. To that end we can sample points from the uniform measure of $F_t$ using Theorem \ref{thm:sampling} above and decide whether this ratio is smaller than $1/4$ (in which case we update) or bigger than $1/2$ (in which case we do note update). Also it is easy to see that whether to update or not when the ratio is in $[1/4, 1/2]$ does not matter for our argument.

In summary, we have the following intermediate result: Excluding the restart condition, each step of the algorithm can be run in at most $\mathrm{poly}(n \log(T)) T$ time.

\subsection{The restart condition: replacing Ellipsoids by Boxes}
To test the restart condition we need to approximate the values $\min_{x \in F_t} \tilde{L_t}(x)$ and $\min_{x \in \partial F_t \cap \mathrm{int}(\cK)} \tilde{L_t}(x)$ at a given time step $t$ and to be able to determine with high probability, whether the difference between these two values is larger than the parameter $\beta$. 

The first observation we can make is that, thanks to Proposition \ref{prop:conc}, upon testing this condition we can always assume that $t \leq \TF$. Consequently, we can rely on the assumption that $\tilde{L}(x)$ is $1/\eta_1$-approximately convex. Minimizing an approximately-convex function over a convex set is a well understood task, see \cite{BLNR15}. Since the set $F_t$ is convex, it is not hard to attain an approximation for $\min_{x \in F_t} \tilde{L_t}(x)$. However, the set $\partial F_t \cap \mathrm{int}(\cK)$ is not convex. This raises an issue which will require us to come up with a slight modification for the construction of the set $F_t$.

Our idea is to replace each ellipsoid
$$\cE_p(r) := \{x \in \R^n : \|x - \mu(p)\|_{\mathrm{Cov}(p)^{-1}} \leq r\}$$
used in the construction of $F_t$ by a respective box, defined as
$$\cB_p(r) := \{x \in \R^n : \|D^{1/2} U (x - \mu(p))\|_{\infty} \leq r \}$$
where $U^\top D U$ is an orthogonal diagonalization of $\mathrm{Cov}(p)$.

Upon doing so, $F_t$ becomes the intersection of $\cK$ with $O(n \log(T))$ many boxes. Since $\cK$ is assumed to be a polytope, so is $F_t$. Therefore, $\partial F_t \cap \mathrm{int}(\cK)$ is the union of polynomially many polytopes. Fix $t \in [\TF]$ and denote by $\cB_1,...\cB_k$ the boxes used to construct $F_t$, hence
$$
F_t = \cK \cap \bigcap_{i \in [k]} \cB_i.
$$
Moreover, denote by $\cF_1,...,\cF_\ell$ the $n-1$-dimensional facets of these boxes. Minimizing $\tilde{L}_t(x)$ over $\partial F_t \cap \mathrm{int}(\cK)$ now amounts to: For each $i \in [\ell]$, check whether the intersection $\cF_i \cap F_t$ is nonempty and, if it is nonempty, minimize $\tilde L_t$ over this convex set. 

Formally, we use the following result:
\begin{theorem} [\cite{BLNR15}]
Fix $\kappa > 0$. Let $\Omega \in \R^n$ be a polytope defined by $m$ linear constraints. Assume that all coefficients in these constraints are rational numbers whose numerators and denominators have absolute values bounded by $M$.\footnote{In the original formulation, the authors assumed that the convex set $\Omega$ is well-rounded by an ellipsoid. One way to find such ellipsoid for a polytope is to use interior point methods. Those algorithms usually produce a Dikin ellipsoid which is a $O(m)$ rounding ellipsoid (or other ellipsoids approximating the domain). If the numerators and denominators coefficients of the polytope are bounded by $M$, one can find a Dikin ellipsoid in $\mathrm{poly}(m \log(D))$ time. See \cite[Appendix E]{lee2013path} for the discussion of the rational polytope assumption and \cite[Section 7.3]{lee2015efficient} for the discussion of the ellipsoid produced by an interior point method.}
Assume that there is a convex function $g$ such that $|f(x)-g(x)|\leq\kappa$ for all $x\in \Omega$. Also, assume that $|f(x)|\leq M$ for all $x\in\Omega$. Then, we can produce a point $x$ with probability $1-\rho$ such that
$$f(x) - \min_{x \in \Omega} f(x) = O(n \kappa)$$
in time $$O \left (\mathrm{poly} \left  (m \log \left (\frac{M}{\rho \kappa}\right ) \right ) \mathrm{Oracle} \right )$$ where $\mathrm{Oracle}$ is the maximum between the times needed to compute $f(x)$ and to check if a point is in $\Omega$ or not.
\end{theorem}

In view of this theorem, we still have to resolve the following three issues that come up:
\begin{enumerate}[label=(\roman*)]
\item \label{issue1}
Since the function $\tilde L_t$ is assumed to be $1/\eta_1$-approximately convex, the above theorem only allows us to approximate its minimum to an error of $O(n/\eta_1)$. However, as currently formulated, the restart condition requires us to check if the two minima differ by an additive factor of $\beta/\eta_1 = 4/\eta_1$.
\item\label{issue3}
Since we replace the ellipsoids $\cE_i$ by the boxes $\cB_i$, this will require a different choice of parameters for the algorithm (which will eventually lead to the worse dependence of the regret on the dimension). Since $\cE_i \subset \cB_i$, we will need a smaller choice of the parameter $\lambda$ in order for the result of Claim \ref{lem:zeta} to remain correct.
\item \label{issue2}
We need to make sure that the boxes $\cB_i$ are defined by constraints whose coefficients are rational numbers with small numerators and denominators.
\end{enumerate}

To deal with \ref{issue1}, we simply choose $\beta$ to be of order $\Theta(n)$, so that it would be enough to have an approximation of the aforementioned values up to that order. This will not change our regret bound: we would get an additional $O(n/\eta_1)$ additive term in \eqref{eq:minmin}, under which equation \eqref{eq:finalest} would remain unchanged, up to the constant term.  \\

Next, we explain how to resolve issue \ref{issue3}. Clearly, we have that $\cE_p(\alpha) \subset \cB_p(\alpha) \subset \cE_p(\sqrt n \alpha)$. The fact that $F_t$ is constructed as the intersection of ellipsoids played a role in the following parts of our proof:

\begin{itemize}
\setlength{\itemindent}{.5in}
\item [(Lemma \ref{lem:largerintersection})] If a convex body has a large intersection with an ellipsoid, that convex body is contained inside a $O(n)$-size larger ellipsoid. It is easy to see the same proof extends to the intersection with symmetry convex bodies.
\item [(Claim \ref{lem:zeta})] This claim gives a bound for the regularity of the function $\tilde \ell_t$ in the ellipsoid $\cE_{p_t}(10 n \alpha)$. In order for the bound to remain true inside the corresponding boxes, we need to change our parameters in a way that allows us to multiply $R_2$ by a factor $\sqrt{n}$. To make sure that $\tilde{\ell}$ is still bounded by a constant, we would need to set $\lambda = \frac{\varepsilon^4}{C^2 n^5 \alpha^2 \log^2(T)}$ instead of $\frac{\varepsilon^4}{C^2 n^4 \alpha^2 \log^2(T)}$.
\item [(Claim \ref{claim:Qbig})] This claim ensures that the function $\mathcal{Q}$ is large outside the ellipsoids $\cE_{i}$. Since the boxes contain those respective ellipsoids, the same bound holds immediately for the new construction of $F_t$.
\end{itemize}
Therefore, in order to be able to replace ellipsoids by boxes, we only need to set $\lambda$ smaller. The rest of the proof remains unchanged (up to the minor changes described above). \\

Issue \ref{issue2} is slightly more involved. In order to resolve it, we fix a grid $\Lambda$ of resolution $T^{-cn}$. We argue that, without affecting the algorithm, one may assume that the set $\cK$ as well as the boxes $\cB_i$ are aligned to $\Lambda$. 

First, remark that we are allowed to replace the boxes $\cB_i = \cB_{p_{\tau_i}}(\alpha)$ by any set $L$ satisfying $\cB_{p_{\tau_i}}(\alpha) \subset L \subset \cB_{p_{\tau_i}}(2 \alpha)$, since result of Claim \ref{lem:zeta} will remain correct upon this modification. The idea is to choose the set $L$ to be a perturbation of the box $\cB_i$ which aligns its vertices to the grid $\Lambda$. This can be done under the assumption that the box $\cB_i$ itself is not too small which, in turn, follows from the fact that the covariance matrix of $p_t$ is bounded from below, as ensured by Claim \ref{claim:detcov}.

\subsection{Summary}
%


\begin{theorem} \label{th:main2}
Assume the domain $\cK$ is a polytope with $\mathrm{poly}(n)$ constraints. Assume that all coefficients in the constraints are rational numbers with absolute values of numerators and denominators bounded by $\mathrm{poly}(T)$. Then the variant of Algorithm \ref{fig:alg} described above satisfies, with probability at least $1-1/T$,
$$R_T \leq O(n^{10.5} \log^{7.5}(T) \sqrt{T}) .$$ Furthermore, each step can be run in $\mathrm{poly}(n \log(T)) T$-time.
\end{theorem}

Note that the cost per each iteration is $\mathrm{poly}(n \log(T)) T$-time, where the factor $T$ comes from the fact that computing $\tilde L_t$ requires us to sum up to $T$ Gaussian functions. To get a slightly better result, one can approximate the functions $\tilde{\ell}$ by their respective Taylor expansions around an arbitrary point in $F_t$. Then, we can store the sum of those Taylor expansions instead of summing every iterations. It can be verified that the $k$ order expansion of $\tilde{\ell}$ at $\mu(p)$ has error $O(n^4 \log^3(T)R_1)^k$, as ensured by Claim \ref{lem:zeta}. By setting $\lambda$ smaller, one can make $R_1$ smaller and hence the expansions converge faster. To make the error of Taylor expansions smaller than $1/\mathrm{poly}(T)$, we need $\log(T)/\log(n^4 \log^3(T)R_1)$ steps and hence it takes $n^{\log(T)/\log(n^4 \log^3(T)R_1)}$ space and time to store and calculate a Taylor expansion. Therefore, we can set $R_1 = n^{-\rho-4} \log^{-3}(T)$ and get an algorithm for sampling in time $\mathrm{poly}(n) T^{O(1/\rho)}$. Since we set $\lambda$ smaller, the regret becomes larger. This is summarized in the following result:

\begin{theorem}
Assume the domain $\cK$ is a polytope with $\mathrm{poly}(n)$ constraints. Assume that all coefficients in the constraints are rational numbers with absolute values of numerators and denominators bounded by $\mathrm{poly}(T)$. For any $\rho > 0$, there is a variant of Algorithm \ref{fig:alg} which satisfies, with probability at least $1-1/T$, 
$$R_T \leq O(n^{10.5 + O(\rho)} \log^{7.5}(T) \sqrt{T}).$$ Furthermore, each step can be run in $\mathrm{poly}(n \log(T)) T^{1/\rho}$-time. In particular, we can attain a regret of at most $n^{O(1)} T^{1/2 + 1/\log\log T}$ in time $\mathrm{poly}(n \log(T))$.
\end{theorem}

\section{Technical lemmas}
We gather here a few technical lemmas on approximately log-concave measures.

\begin{lemma} \label{lem:EllipsoidInSupport}
Let $q(x)$ be an $(1/e)$-approximately log-concave probability measure on $\RR^n$. Then
$$
\cE_q(1/100) \subset \mathrm{Supp}(q) 
$$
\end{lemma}
\begin{proof}	
	By applying a linear tranformation, we can clearly assume without loss of generality that $q$ is isotropic. Let $g(x)$ be a log-concave probability measure satisfying $e q \leq a g \leq e^{-1} q$ for a normalization constant $1/e<a<e$. Let $S$ be the support of $q$. Since $S$ is also the support of $g$, it is clearly convex. Assume without loss of generality that there exists $x \notin S$ with $|x| \leq 1/100$. By the Hahn-Banach theorem, there is a hyperplane separating $x$ from $S$, in other words, there exists $\theta$ with $|\theta| = 1$ such that $S \subset \{y; \langle y, \theta \rangle \leq 1/100 \}$. Define
	$$
	\tilde q(t) = \int_{\theta^\perp} q(\theta t + y) dy
	$$
	the marginal of $q$ onto the direction $\theta$ and likewise let $\tilde g(t)$ be the respective marginal of $g$. By Prekopa-Leindler, we have that $\tilde g$ is log-concave. Using Lemma \ref{lem:comp} we have that $\sqrt{\Var[\tilde{g}]} \geq 1/e$. By \cite[Lemma 5.5]{LV06}, we have that $\tilde g(t) \leq e$ for all $t$, and consequently $\tilde q(t) \leq a e \leq e^2$ for all $t$. Since $\tilde q$ is supported on $(-\infty, 1/100)$, and since it is centered, we have that
	$$
	- \int_{-\infty}^0 x \tilde{q}(x) dx = \int_0^{1/100} x \tilde{q}(x) dx \leq \int_0^{1/100} x e^3 dx = 2 e^3 (100)^{-2}.
	$$ 
	However, since $\int \tilde q = 1$, we have that 
	$$
	\int_{-\infty}^{-\frac{1}{2 e^3}} \tilde q(x) dx \geq 1 - (\frac{1}{2 e^3} + \frac{1}{100}) e^3 \geq 1/4,
	$$
	which implies that 
	$$
	- \int_{-\infty}^0 x \tilde{q}(x) dx \geq \frac{1}{2 e^3} \cdot \frac{1}{4} \geq 2 e^3 (100)^{-2}.
	$$
	We reach a contradiction and the proof is complete.
\end{proof}

\begin{lemma} \label{lem:GaussianTail}
Fix a dimension $n \geq 1$ and an integer $T \geq 10$. Let $X$ be a Gaussian vector in $\RR^n$ distributed according to the law $\cN(0, \Theta)$. Let $f:\RR^n \to [0, \infty)$ be a function satisfying $f(x) = 0$ on $\{\langle x, \Theta^{-1} x \rangle \leq 20 n \log T  \}$ and $f(x) \leq T |x| + 2$ on $\RR^n$. Then
$$
\E f(X) \leq \frac{ \Vert \Theta \Vert_{\mathrm{OP}}^{1/2}+1}{T^3}.
$$
\end{lemma}

\begin{proof}
First note that,
\begin{align}
\E \left [ |X| \ds 1\{\langle X, \Theta^{-1} X \rangle > 20 n \log T \} \right ] ~& = \E \left [ |\Theta^{1/2} Z| 1\{|Z|^2 > 20 n \log T \} \right ] \nonumber \\
& \leq \Vert \Theta \Vert_{\mathrm{OP}}^{1/2} \E \left [ |Z| 1\{|Z|^2 > 20 n \log T \} \right ] \label{eq:eqtail1}
\end{align}
where $Z$ is a standard Gaussian vector. A well-known concentration estimate for Gaussian measures states that for a $1$-Lipschitz function $\varphi$ one has that
$$
\P( \varphi(Z) \geq \E[\varphi(Z)] + t) \leq 2 \exp(-t^2/2).
$$
Since we have $\E[|Z|] \leq \sqrt{\E[|Z|^2]} = \sqrt{n}$, the above gives
\begin{equation}\label{eq:concnorm}
\P ( |Z| \geq \sqrt{n} + t ) \leq 2 \exp(-t^2/2).
\end{equation}

Consequently we have, using integration by parts,	
\begin{align}
\E \bigl [ |Z| \ds 1\{|Z|^2 > 20 n \log T \} \bigr ] ~& = \int_{\sqrt{20 n \log T}}^\infty \P( |Z| > t) dt \nonumber \\
& \stackrel{\eqref{eq:concnorm}}{\leq} 2 \int_{\sqrt{20 n \log T} - \sqrt{n}}^{\infty} \exp(-s^2/2) ds \nonumber \\
& \leq 2 \int_{\sqrt{12 \log T}}^ \infty \exp(-s^2/2) ds \leq \frac{1}{2 T^4}. \label{eq:eqtail2}
\end{align}	
Finally,
\begin{align*}
\E f(X) ~& \leq \E \Bigl [(T |X| + 2) \ds 1 \{\langle X, \Theta^{-1} X \rangle \geq 20 n \log T  \} \Bigr ] \\
& \stackrel{\eqref{eq:eqtail1}}{\leq} T \Vert \Theta \Vert_{\mathrm{OP}}^{1/2} \E \left [ |Z| 1\{|Z|^2 > 20 n \log T \} \right ] + 2 \P(|Z|^2 > 20 n \log T) \\
& \stackrel{\eqref{eq:eqtail2} \wedge \eqref{eq:concnorm}}{\leq} \frac{ \Vert \Theta \Vert_{\mathrm{OP}}^{1/2}+1}{T^3}.
\end{align*}
\end{proof}	

For a non-negative density $f(x)$ on $\RR$, denote
$$
\EE[f] = \frac{\int_\RR x f(x) dx}{\int_{\RR} f(x) dx}, ~~ \Var[f] = \frac{\int_\RR x^2 f(x) dx}{\int_{\RR} f(x) dx} - \EE[f]^2.
$$

\begin{lemma} \label{lem:comp}
	Let $f(x), g(x)$ be two non-negative functions such that $\int (x^2 + 1) f(x) dx < \infty$ and such that $\eps < g(x) / f(x) < 1/\eps$ for some $\eps \in (0,1)$. Then 
	$$
	\bigl \vert \E[f] - \E[g] \bigr \vert < \frac{\sqrt{\Var[f]}}{2 \eps^2}
	$$
	and
	$$
	\eps^2 \leq \frac{\Var[g]}{\Var[f]}  \leq \frac{1}{\eps^2}.
	$$
\end{lemma}
\begin{proof}
	We can clearly assume without loss of generality that $\int_\RR f(x) = 1$ and $\int x f(x) dx = 0$. We have
	$$
	\int x g(x) dx \leq \frac{1}{\eps} \int_0^\infty x f(x) dx = \frac{1}{2 \eps} \int_\RR |x| f(x) dx \leq \frac{1}{2 \eps} \sqrt{\int_\RR x^2 f(x) dx}.
	$$
	We therefore have $\E[g] = \frac{\int x g(x) dx}{\int g(x) dx} \leq \tfrac{1}{2 \eps^2} \sqrt{\Var[f]}$ which completes the first part by symmetry. For the second part, we remark that 
	$$
	\Var[g] \leq \frac{\int_{\RR} x^2 g(x) dx}{\int_{\RR} g(x) dx} \leq \frac{ \Var[f] }{\eps^2}
	$$
	and the reverse inequality follows by a similar argument.
\end{proof}


\begin{lemma} \label{lem:tail}
Let $n \geq 2$. Let $f(x)$ be an isotropic log-concave density on $\RR^n$. Then for all $x \in \RR^n$ with $|x| \geq e^{15} n \log n$, one has that
\begin{equation}
f(x) < \exp  \left ( - \frac{ |x|}{e^{15}} \right ).
\end{equation}
\end{lemma}
\begin{proof}
Define $H = x^\perp$. According to \cite[Lemma 5.5(b)]{LV07} and via an application of the Pr\'ekopa-Leinder inequality, we have that
\begin{equation} \label{eq:hyperplane}
\int_H f(y) dy \geq \frac{1}{8}.
\end{equation}
Moreover, we have the bound (\cite[Theorem 5.14(e)]{LV07})
\begin{equation} \label{eq:upperisop}
f(y) \leq e^{6n + n \log n}, ~~ \forall y \in \RR^n.
\end{equation}
Define $\theta = x/|x|$. We can estimate
\begin{align*}
1 = \int_{\RR^n} f(y) dy ~& \geq \int_0^{|x|} \int_{H} f(t \theta + w) dw dt \\
&= |x| \int_0^1 (1-t)^{n-1} \int_{H} f(t x + (1-t) w) dw dt \\
& \geq |x| \int_0^{1/(n \log n)} (1-t)^{n-1} \int_{H} f(x)^t f(w)^{1-t} dw dt \\
& \geq |x| \int_0^{1/(n \log n)} f(x)^t (1-t)^{n-1} \left (\max_{y \in H} f(y)\right )^{-\tfrac{1}{n \log n}} \int_{H} f(w) dw dt \\
& \stackrel{ \eqref{eq:hyperplane}, \eqref{eq:upperisop}}{\geq} e^{-12} |x| \int_0^{1/(n \log n)} f(x)^t dt \\
& \geq e^{-12} |x| \left ( \frac{\mathbf{1}_{ \{f(x) \geq 2^{-n \log n}\} }}{2 n \log n} + \frac{\mathbf{1}_{ \{f(x) < 2^{-n \log n}\} }}{-2 \log f(x)} \right ).
\end{align*}
The last inequality implies that whenever $|x| \geq e^{15} n \log n$, one has that $- \log f(x) \geq e^{-15} |x|$, and the proof is complete.
\end{proof}

\begin{lemma} 
Let $f(x) = \frac{1}{Z} \exp(-V(x))$ be $\eps$-approximately log-concave with $0<\eps<1/2$. Assume that $\min_{x \in \RR^n} V(x) = 0$. Then,
\begin{equation}\label{eq:logZBound}
		- n (\log n + 8) + 2 n \log \eps + \tfrac{1}{2} \log \det \COV(f) \leq \log Z \leq 5n (1 - \log \eps) + \tfrac{1}{2} \log \det \COV(f).
\end{equation}
\end{lemma}

\begin{proof}
Let $g(x)$ be a log-concave function such that 
$$
\eps g(x) \leq f(x) \leq \tfrac{1}{\eps} g(x).
$$
As a consequence of Lemma \ref{lem:comp}, we have 
$$\Vert \mu(g) - \mu(f) \Vert_{\COV(g)^{-1}} \leq 1 / \eps^2$$
and
$$\eps^2 \COV(f) \preceq \COV(g) \preceq  \frac{1}{\eps^2} \COV(f)$$
(in the positive definite sense) which implies that, for all $x \in \RR^n$,
\begin{equation} \label{eq:compdist}
\Vert x - \mu(g) \Vert_{\COV(g)^{-1}} \geq \eps \Vert x - \mu(f) \Vert_{\COV(f)^{-1}} - \frac{1}{\eps^2}
\end{equation}
and also that
\begin{equation}\label{eq:compdet}
\eps^{2n} \det \COV(f) \leq \det \COV(g) \leq \eps^{-2n} \det \COV(f).
\end{equation}

Combining the bound \eqref{eq:compdet} with the bound (\cite[Theorem 5.14(c)]{LV07})
$$
g(x) \geq (4e\pi)^{-n} \det \COV(g)^{-1/2}
$$
gives that there exists a point $x \in \RR^n$ such that
$$
f(x) \geq \eps g(x) \geq \exp(-5n) \eps^{n+1} \det \COV(f)^{-1/2}.
$$
For the other side, we use the bound (\cite[Theorem 5.14(e)]{LV07})
$$
g(x) \leq e^{6n + n \log n} \det \COV(g)^{-1/2}, ~~ \forall x \in \RR^n,
$$
which, combined with \eqref{eq:compdet} gives
$$
f(x) \leq \eps^{-n-1}  e^{6n + n \log n} \det \COV(f)^{-1/2}, ~~ \forall x \in \RR^n,
$$
By assumption, we have $\log(Z) = - \max_{x \in \RR^n} \log f(x)$ which finishes the proof.
\end{proof}

\begin{lemma} \label{lem:covtovalue}
Let $f(x) = \frac{1}{Z} \exp(- V(x))$ be $\epsilon$-approximately log-concave with $0<\eps<1/2$. Assume that $\min_{x \in \RR^n} V(x) = 0$. Then one has:
$$\|x-\mu(f)\|_{\mathrm{Cov}(f)^{-1}} \leq \frac{\exp(15)}{\epsilon^2} \left( V(x)-V^* + \frac{1}{\epsilon^2} + 7 n (1+\log(n / \epsilon)) \right) .$$
\end{lemma}

\begin{proof}
An application of Lemma \ref{lem:tail} combined with the fact that $\eps \int_{\RR^n} g(x) dx \leq 1$, gives that for all $x \in \RR^n$ with $\Vert x - \mu(g) \Vert_{\COV(g)^{-1}} \geq e^{15} n \log n$, one has
$$
g(x) \leq \frac{1}{\eps} \exp \left (- e^{-15} \Vert x - \mu(g) \Vert_{\COV(g)^{-1}} \right ) \det \COV(g)^{-1/2}.
$$
Together with equations \eqref{eq:compdist} and \eqref{eq:compdet}, this gives that 
\begin{align*}
& f(x) \leq \eps^{-n-1} \exp \left (-e^{-15} \eps \Vert x - \mu(f) \Vert_{\COV(f)^{-1}} + \frac{1}{\eps^2} \right ) \det \COV(f)^{-1/2}.
\end{align*}
Now, by assumption and \eqref{eq:logZBound}, we have
$$
V(x) = - \log f(x) - \log Z \geq - \log f(x) - 5n (1 - \log \eps) - \tfrac{1}{2} \log \det \COV(f).
$$
Combining the last two bounds gives
$$
V(x) \geq - 7n (1 - \log \eps) + e^{-15} \eps \Vert x - \mu(f) \Vert_{\COV(f)^{-1}} - \frac{1}{\eps^2}
$$
which finishes the proof.
\end{proof}

\begin{lemma} \label{lem:firstmoment}
	Let $f(x)$ be an isotropic $\eps$-approximately log-concave density on $\RR$ with $0<\eps<1/2$, then 
	\begin{equation}
	\int_{s}^\infty (x-s) f(x) dx > \frac{\eps}{80}
	\end{equation}
	whenever $s \leq \tfrac{\eps}{80}$.
\end{lemma}
\begin{proof}
By the $\eps$-approximate log-concavity assumption, there exists a log-concave function $g(x)$ with $\eps g(x) \leq f(x) \leq g(x) / \eps$ for all $x$. Define $\tilde g(x) = \tfrac{g(x)}{\int_{\RR} g(x) dx}$. Let $X,Y$ be random variables with densities $f, \tilde g$ respectively. According to Lemma \ref{lem:comp} we have that
\begin{equation} \label{eq:VEY}
\eps^2 \leq \Var[Y] \leq \frac{1}{\eps^2}, ~~ \left |\EE[Y] \right | \leq \frac{1}{2\eps^2}.
\end{equation}
According to \cite[Lemma 5.7]{LV07} we have 
$$
\PP \left (\left | \frac{Y - \EE[Y]}{\sqrt{\Var[Y]}} \right | > t \right ) < e^{1-t}, ~~ \forall t \in \RR.
$$
Now, according to \eqref{eq:VEY}, we have that
$$
|Y| \geq \frac{t}{\eps} + \frac{1}{2 \eps^2} \Rightarrow |Y - \EE[Y]| \geq t \sqrt{\Var[Y]}
$$
which in turn gives that
$$
\PP(|Y| > t) \leq \exp \left (1 + \frac{1}{2\eps} - t \right ), ~~ \forall t \in \RR
$$
and consequently,
\begin{equation} \label{eq:estx}
\PP(|X| > t) \leq \frac{1}{\eps^2} \exp \left (1 + \frac{1}{2 \eps} - t \right ), ~~ \forall t>0.
\end{equation}
Now, by the isotropicity of $X$, we have
$$
\int_0^\infty 2t P(|X|>t) dt = \EE \left [X^2 \right] = 1.
$$
Moreover, equation \eqref{eq:estx} gives for $s>1$,
\begin{align*}
\int_s^\infty 2t P(|X|>t) dt ~& \leq \tfrac{1}{\eps^2} \exp(1+1/(2\eps)) \int_s^\infty 2 t \exp \left (- t \right ) dt  \\
& = \tfrac{2 e}{\eps^2} \exp(1/(2\eps)) \int_{s}^\infty w \exp (-w) dw \\
& \leq \exp(3/\eps) \int_{s}^\infty w \exp (-w) dw \\ 
& = \exp(3/\eps) (s+1) \exp(- s). 
\end{align*}
Taking $s = 5 / \eps$ and combining with the previous display, we get
$$
\int_{0}^{5 / \eps} t P(|X|>t) dt \geq 1/4.
$$
So we have that
$$
\int_0^\infty x f(x) dx = \frac{1}{2} \EE \left [|X| \right ] = \frac{1}{2} \int_0^\infty P(|X|>t) dt \geq \frac{\eps}{40}.
$$
Combining this with the fact that $\int_\RR f(x) dx = 1$, we have for $s \leq \frac{\eps}{80}$ that
$$
\int_s^\infty (x - s) f(x) dx \geq \int_0^\infty x f(x) dx - s \geq \frac{\eps}{80},
$$
which finishes the proof.
\end{proof}

\begin{lemma} \label{lem:onelastone}
Let $f$ be isotropic, $\eps$-approximately log-concave with $0<\eps<1/2$. Let $\theta \in \Sph$ and $y \in \RR^n$ with $|y| < \eps / 80$. Defining
		$$
		h_\theta(x) = \max( \langle x - y, \theta \rangle, 0  )
		$$
		we have that
		$$
		\int h_\theta(x) f(x) dx \geq \frac{\eps}{80}.
		$$
\end{lemma}

\begin{proof}
By taking the marginal onto the direction $\theta$, the claim clearly becomes one-dimensional. The result is now a direct consequence of Lemma \ref{lem:firstmoment}.
\end{proof}

%

\begin{lemma} \label{lem:convexdom}
	For any $1/2>\epsilon>0$ and any isotropic $\epsilon$-approximately log-concave measure $p$, and any measure $r$ with $\EE[r] = 0$ and support included in $\{x \in \RR^n : |x| \leq \eps / 80\}$, one has that $r$ is convexly dominated by $p$. 
\end{lemma}
\begin{proof}
Fix a convex test function $\varphi$. Our goal is to prove that $\int \varphi r \leq \int \varphi p$. Since both densities are centered, we may add any linear function to $\varphi$ without affecting this inequality, so we may legitimately assume that $\varphi(0) = 0$ and that $\varphi(x) \geq 0$ for all $x \in \RR^n$. Define $D = \{x \in \RR^n : |x| \leq a\}$ with $a$ being a constant chosen later on. Define $x_m = \arg \max_{x \in D} \varphi(x)$ and $m = \varphi(x_m)$. By the assumption $\varphi(0) = 0$, and by the convexity of $\varphi$, we have that $|\nabla \varphi(x_m)| \geq m/a$. Consequently, using the assumption $\varphi \geq 0$ we conclude that
$$
\varphi(x) \geq \frac{m}{a} \max(0, \langle x - x_m, \theta \rangle)
$$ 
where $\theta = \tfrac{\nabla \varphi(x_m)}{|\nabla \varphi(x_m)|}$. An application of Lemma \ref{lem:onelastone} thus teaches us that, under the assumption $|x_m| \leq \eps / 80$, we have
$$
\int_{\RR^n} \varphi(x) p(x) dx \geq \frac{m}{80 a} \eps.
$$
Thus choosing $a = \eps / 80$, we have 
$$
\int_{\RR^n} \varphi(x) p(x) dx \geq m \geq \int_{\RR^n} \varphi(x) r(x) dx
$$
which completes the proof.
\end{proof}

\bibliographystyle{plainnat}
\bibliography{newbib}

\begin{algorithm}[t]
\begin{algorithmic}[1]
\State \textbf{Parameters:} $\lambda \in (0,1)$, $\sigma^2 > 0$, $\eta_1 > 0$, $\alpha >0$, $\gamma > 0$

\Comment{To prove Theorem \ref{th:main} we take the following scaling of these parameters: $\lambda = \Theta(\frac{1}{n^8 \log^6(T)})$, $\sigma^2 = \Theta(\frac{1}{n \log(T)})$, $\eta_1 = \Theta(\frac{1}{\sqrt{n T \log(T)}})$, $\alpha = \Theta(n^2 \log^2(T))$, $\gamma=\Theta(\frac{1}{n \log(T)})$}
\State \textbf{Initialization:}
\State For all $x \in \cK$, $p(x) \gets \frac{\ds1\{x \in \cK\}}{\mathrm{vol}(\cK)}$ \Comment{$p$ will be the exponential weights strategy}
\State For all $x \in \cK$, $\tilde{L}(x) \gets 0$ \Comment{Cumulative loss estimate}
\State $\eta \gets \eta_1$ \Comment{The learning rate $\eta$ will be adaptative and time-dependent}
\State $F \gets \cK$ \Comment{$F$ will be the focus region of the algorithm}
\State \textbf{Notation:} Denote $\mu(p)$ and $\mathrm{Cov}(p)$ for the mean and covariance of $p$, $\cE_p(r) = \{x \in \R^n : (x-\mu(p))^{\top} \mathrm{Cov}(p)^{-1} (x-\mu(p)) \leq r^2\}$, and $\Phi_{\mu, \Sigma}$ for the density of a Gaussian with mean $\mu$ and covariance $\Sigma$.
\For{$t=1,\hdots,T$} \Comment{Main loop}
\State Draw $X$ at random from $p$ \Comment{Draw a point from the exponential weights}
\State Draw $C$ at random from $\cN(\mu(p), \sigma^2 \lambda \mathrm{Cov}(p))$ \Comment{Draw a point from the Gaussian core of $p$}
\State Play $x_t = \lambda X + (1-\lambda) C$ \Comment{Play an interpolation of the two above points}
\State Receive loss $\ell=\ell_t(x_t)$ (if $x_t \not\in \cK$ set $\ell=0$) \Comment{Suffer loss}
\State $u \gets p' \ast c' (x_t)$ where $p'(x) = \frac{1}{\lambda} p(x/\lambda)$ and $c'(x) = \frac{1}{1-\lambda} \Phi_{\mu(p), \sigma^2 \lambda \mathrm{Cov}(p)}(x/(1-\lambda))$ \Comment{$u$ is morally the ``probability" of playing $x_t$}
\State For all $x \in \cK$, $\tilde{\ell}(x) \gets \frac{\ell}{u} \Phi_{\mu(p), \sigma^2 \lambda \mathrm{Cov}(p)} \left(\frac{x_t - \lambda x}{1-\lambda} \right)$ \Comment{Loss estimate}
\State For all $x \not\in F$, $\tilde{\ell}(x) \gets +\infty$ \Comment{Loss truncated outside the focus region}
\State For all $x \in \cK$, $\tilde{L}(x) \gets \tilde{L}(x) + \tilde{\ell}(x)$ \Comment{Update the cumulative loss estimate}
\State For all $x \in \cK$, $p(x) \gets \frac{1}{Z} p(x) \exp(- \eta \tilde{\ell}(x))$ where $Z$ is a normalization constant so that $p$ is a density. \Comment{Update of the exponential weights}
\If{$\mathrm{vol}(F \cap \cE_{p}(\alpha)) \leq \frac{1}{2} \mathrm{vol}(F)$} \Comment{Test if focus region should be updated}
\State $F \gets F \cap \cE_{p}(\alpha)$ \Comment{Focus region updated}
\State $\eta \gets (1+\gamma) \eta$ \Comment{Learning rate increased}
\EndIf
\If{$\min_{x \in \partial F \cap \mathrm{int}(\cK)} \tilde{L}(x) - \min_{x \in F} \tilde{L}(x) \leq 2 / \eta_1$} \Comment{Test if there is a point on the boundary of the focus region which is abnormally good}
\State Restart the algorithm
\EndIf
\EndFor
\end{algorithmic}
\caption{Pseudo-code for the high-dimensional strategy} \label{fig:alg}
\label{alg}
\end{algorithm}
\end{document}